\documentclass[sigconf]{acmart}

\copyrightyear{2026}
\acmYear{2026}
\setcopyright{cc}
\setcctype{by}
\acmConference[WSDM '26] {Proceedings of the Nineteenth ACM International Conference on Web Search and Data Mining}{February 22--26, 2026}{Boise, ID, USA.}
\acmBooktitle{Proceedings of the Nineteenth ACM International Conference on Web Search and Data Mining (WSDM '26), February 22--26, 2026, Boise, ID, USA}
\acmDOI{10.1145/3773966.3777927}
\acmISBN{979-8-4007-2292-9/2026/02}

\usepackage{algorithmic}
\usepackage[ruled,linesnumbered]{algorithm2e}
\usepackage{amsmath, amsthm} 
\usepackage{cleveref}
\usepackage{amsfonts}
\usepackage{enumitem}
\newtheorem{definition}{Definition}

\newtheorem{theorem}{Theorem}
\newtheorem{lemma}{Lemma}

\newtheorem{proposition}{Proposition}

\usepackage{bigstrut}
\usepackage{multirow}
\usepackage{array}
\usepackage{tabularx}
\usepackage{booktabs} 
\usepackage{pifont}
\usepackage{subfigure}
\usepackage{caption}

\usepackage{wrapfig}

\newcommand{\Mat}{\boldsymbol}

\newcommand{\real}{\mathbb{R}}
\newcommand{\complex}{\mathbb{C}}
\newcommand{\FT}{\mathcal{F}}

\newcommand{\DC}[1]{\mathcal{DC}\left[ {#1} \right]}
\newcommand{\HC}[1]{\mathcal{HC}\left[ {#1} \right]}

\newcommand{\cmark}{\ding{52}}
\newcommand{\ccross}{\ding{56}}

\begin{document}

\title{ParaFormer: A Generalized PageRank Graph Transformer for Graph Representation Learning}

\author{Chaohao Yuan}
\affiliation{
  \institution{Tsinghua Shenzhen International Graduate School, Tsinghua University}
  \city{Shenzhen}
  \country{China}
}
\affiliation{
  \institution{The Chinese University of Hong Kong}
  \city{Hong Kong S.A.R.}
  \country{China}
}
\email{chaohaoyuan@link.cuhk.edu.hk}

\author{Zhenjie Song}
\author{Ercan Engin Kuruoglu}
\affiliation{
  \institution{Tsinghua Shenzhen International Graduate School, Tsinghua University}
  \city{Shenzhen}
  \country{China}
}
\email{songzj23@mails.tsinghua.edu.cn}

\email{kuruoglu@sz.tsinghua.edu.cn}

\author{Kangfei Zhao}
\author{Yang Liu}
\affiliation{
  \institution{The Chinese University of Hong Kong}
  \city{Hong Kong S.A.R.}
  \country{China}
}
\email{zkf1105@gmail.com}

\email{yliuweather@gmail.com}

\author{Deli Zhao}
\affiliation{
  \institution{DAMO Academy, Alibaba Group, Hupan Lab}
  \city{Hangzhou}
  \country{China}}
\email{zhaodeli@gmail.com}

\author{Hong Cheng}
\affiliation{
  \institution{The Chinese University of Hong Kong}
  \city{Hong Kong S.A.R.}
  \country{China}
}
\email{hcheng@se.cuhk.edu.hk}

\author{Yu Rong}
\authornote{Corresponding Author.}
\affiliation{
  \institution{DAMO Academy, Alibaba Group, Hupan Lab}
  \city{Hangzhou}
  \country{China}}
\email{yu.rong@hotmail.com}

\renewcommand{\shortauthors}{Yuan et al.}

\begin{abstract}
Graph Transformers (GTs) have emerged as a promising graph learning tool, leveraging their all-pair connected property to effectively capture global information. To address the over-smoothing problem in deep GNNs, global attention was initially introduced, eliminating the necessity for using deep GNNs. However, through empirical and theoretical analysis, we verify that the introduced global attention exhibits severe over-smoothing, causing node representations to become indistinguishable due to its inherent low-pass filtering. This effect is even stronger than that observed in GNNs.
To mitigate this, we propose \textbf{Pa}ge\textbf{Ra}nk Trans\textbf{former} (ParaFormer), which features a PageRank-enhanced attention module designed to mimic the behavior of deep Transformers. We theoretically and empirically demonstrate that ParaFormer mitigates over-smoothing by functioning as an adaptive-pass filter.
Experiments show that ParaFormer achieves consistent performance improvements across both node classification and graph classification tasks on 11 datasets ranging from thousands to millions of nodes, validating its efficacy. The supplementary material, including code and appendix, can be found in https://github.com/chaohaoyuan/ParaFormer.
\end{abstract}



\maketitle

\section{Introduction}
\label{sec: introduction}

The graph is a fundamental data structure for representing complex relationships in real-world systems~\cite{largegraph-app-3, largegraph-app-2}, ranging from social networks, molecular interactions to recommendation systems and knowledge graphs. 
In recent years, Graph Neural Networks (GNNs)~\cite{GCN17, GAT, GCNII-20, h2gcn-20, APPNP18, GPRGNN21, yuan2025non} have been a prevailing deep learning framework for graph learning, demonstrating impressive performance in various tasks, e.g., node classification and graph classification. 
However, general GNNs focus on capturing local structural patterns through a neighborhood aggregation paradigm, inherently limiting their ability to model global topological characteristics and long-range dependencies between distant nodes~\cite{wu2021representing, dwivedi2022long}. 
This shortcoming becomes a performance bottleneck in scenarios requiring holistic graph understanding, such as graph-level property prediction~\cite{Graphformer21} or reasoning over multi-hop relational paths~\cite{bambergermeasuring}.
To bridge this gap, researchers have proposed Graph Transformers~\cite{graph-transformer-survey-architecture-22, yuan2025survey} that synergistically integrate topological information with the global attention mechanisms of vanilla Transformers~\cite{Transformer}. 

However, existing Graph Transformers suffer from a significant limitation: \textit{over-smoothing}. This phenomenon
results in node representations becoming nearly indistinguishable, and generally occurs when the pairwise L2 distance between representations falls below a specific threshold. The intrinsic cause lies in the spectral properties of these models: Transformers act as low-pass filters~\cite{attention-oversmoothing-vit}, suppressing high-frequency signals (i.e., information that changes drastically) within the graph structure.  
Unfortunately, while low-frequency signals often capture principal features in vision and language data, high-frequency signals in graphs are frequently crucial, particularly for tasks like node classification in heterophilic graphs~\cite{zhu2020beyond}. 
Consequently, the inherent low-pass nature of Transformers severely hampers their generalization ability across various graph topologies. This fundamental limitation leaves existing Graph Transformers facing a dilemma: they cannot effectively leverage deep architectures – a key factor in the success of Transformers in natural language processing~\cite{bert-2019, radford2019language}, computer vision~\cite{vit2021,DBLP:conf/iclr/LiuZCTZ0L25} and biological applications~\cite{grover20,yuan2026transformer, yuan2025annotation} – without succumbing to detrimental over-smoothing on graph data.

Thus, a significant challenge is raised: \textit{How to effectively model global topological characteristics and long-range dependencies while simultaneously mitigating over-smoothing?} Although regularization or normalization techniques like DropEdge~\cite{Dropedge-20,10195874} and PairNorm~\cite{Zhao2020PairNorm} enable stacking deeper GNN layers, these architectures fundamentally rely on the local message-passing mechanism, inherently limiting their ability to capture global information, especially in very large graphs. Furthermore, even residual connections, another common approach, have been theoretically proven ineffective against over-smoothing in graph learning~\cite{huang2022tacklingoversmoothinggeneralgraph,attention-based-oversmoothing-23}. Graph Transformers are initially speculated to not suffer from over-smoothing. However, a deep theoretical understanding to verify this has been lacking.

In this work, we first show that Transformers exhibit an even stronger over-smoothing tendency than GNNs (i.e., Proposition~\ref{prop:comparsion}), which stems from the fully-connected attention mechanism in Transformers compared to the sparse connectivity in GNNs. Inspired by the PageRank~\cite{PageRank-98} ability of capturing multi-hop relations, we propose a novel Graph Transformer architecture, \textbf{Pa}ge\textbf{Ra}nk Trans\textbf{Former} (ParaFormer). It incorporates the message passing mechanism of Generalized PageRank (GPR) into the attention computation, called Generalized PageRank Attention (GPA), preserving the distinction of node representations in the propagation through deep attention blocks by sharing the weights of these attention blocks. We theoretically show that ParaFormer performs adaptively as a low-pass and a high-pass filter, and mitigates the over-smoothing issue. 
In addition, to reduce the computational complexity of the power iteration of GPR, we devise a scalable GPA with a linear time complexity, which is a close approximation of the original GPA.
We extensively evaluate ParaFormer on 11 node classification and 2 graph classification datasets, including both homophily and heterophily graphs with sizes ranging from thousands to millions of nodes. The datasets cover citation networks, Wikipedia-based graphs, social networks, and recommendation graphs. Results show that ParaFormer outperforms state-of-the-art GNNs and Graph Transformers.

Our main contributions are summarized as follows:
\begin{itemize}[leftmargin=*]
    \item We theoretically and empirically demonstrate that the over-smoothing issue substantially undermines the effectiveness of deep Transformers in the field of graph learning. 
    \item We propose ParaFormer, an innovative GT that integrates GPR into the Transformer architecture. By serving as a learnable filter, ParaFormer resists over-smoothing while preserving long-range and multi-hop modeling capabilities. 
    \item To improve scalability, we design a scalable attention mechanism for ParaFormer, enabling a linear computational complexity without sacrificing model performance. 
    \item Through comprehensive experimental evaluations across diverse graph datasets, we validate that ParaFormer effectively alleviates over-smoothing and achieves superior performance in node-level and graph-level classification tasks.
\end{itemize}

\section{Preliminaries}
A graph can be represented as $\mathcal G = (\mathcal V, \mathcal E)$, where $\mathcal V$ denotes the set of nodes and $\mathcal E$ denotes the set of edges. The node set $\mathcal V$ comprises $n$ nodes, associated with a feature matrix $\mathbf{X} \in \mathbb R^{n \times d}$ and a label matrix $\mathbf{Y} \in \mathbb R^{n \times c}$, where $d$ and $c$ represent the dimension of the node features and the number of classes, respectively. The edge set $\mathcal E$ can define an adjacency matrix $\Mat{A} \in \mathbb R^{n \times n}$, where $\Mat{A}_{u,v} = 1$ if there exists an edge for node pair $(u, v)$ in $\mathcal E$, and $\Mat{A}_{u,v} = 0$ otherwise. 

The primary objective of graph learning is to learn the node representations $z$ that can be used for various downstream tasks, such as graph classification and node classification tasks. 

\paragraph{\textbf{Message-Passing Graph Neural Networks.}}
The core mechanism of mainstream Graph Neural Network (GNN) methods is the message passing mechanism, which aggregates the embeddings of neighboring nodes. This process can be described by the following framework:

\begin{equation}
\begin{split}
        z_{u}^{(k)} = f^{(k)}(z_{u}^{(k)}), z_{u} = \sum_{v \in N(u)} h^{(k)}(z_v),
\label{message passing}
\end{split}
\end{equation}
where $k$ represents the $k$-th layer in the GNN, $N$ denotes the receptive field in this GNN, which normally consists of the neighbors, and $h$ indicates the aggregation function of this GNN.

\begin{figure*}[t]
    \centering
    \includegraphics[width=0.24\textwidth]{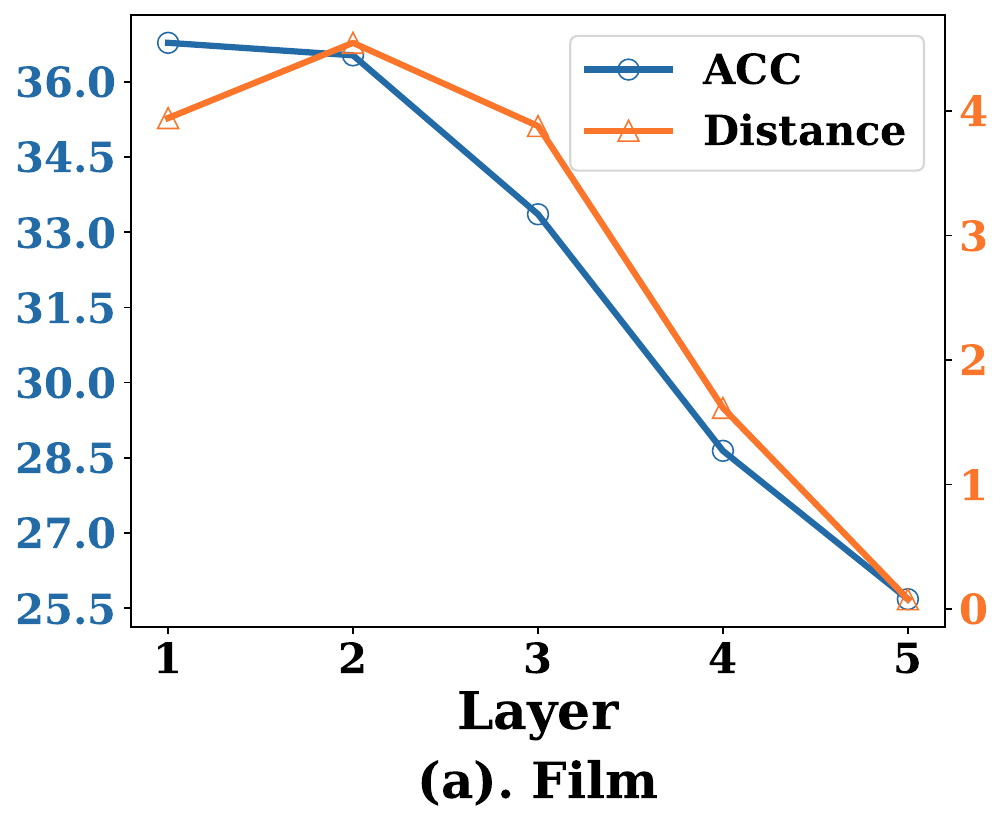}
    \includegraphics[width=0.24\textwidth]{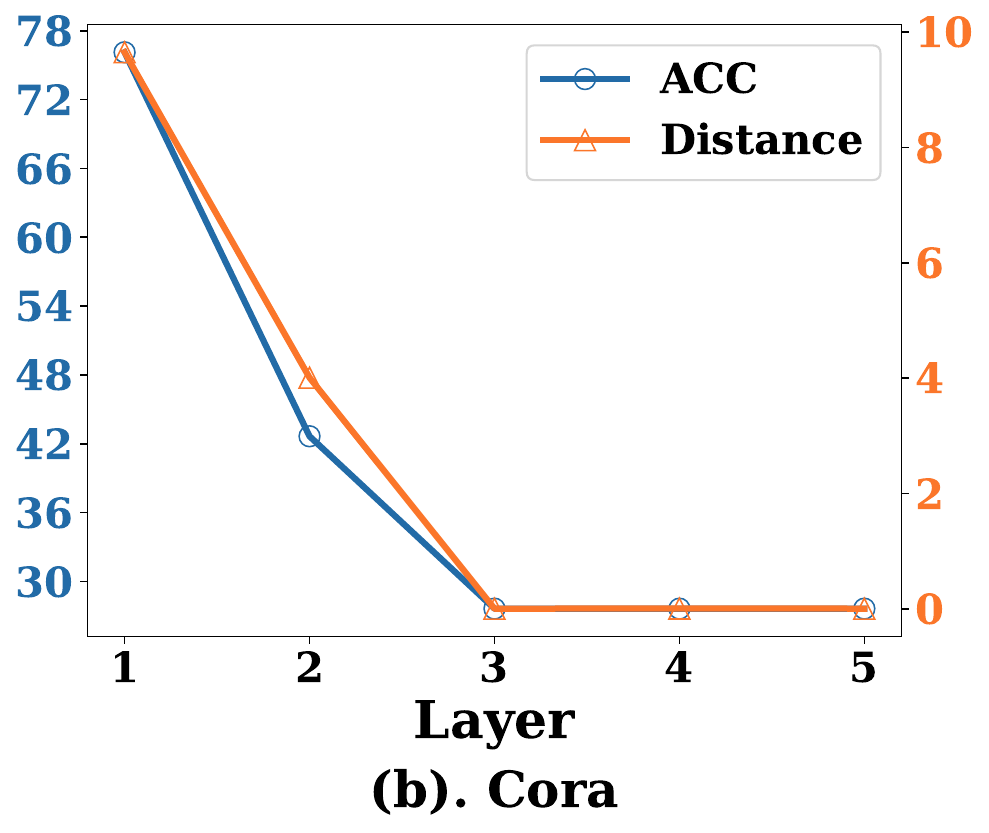}
    \includegraphics[width=0.24\textwidth]{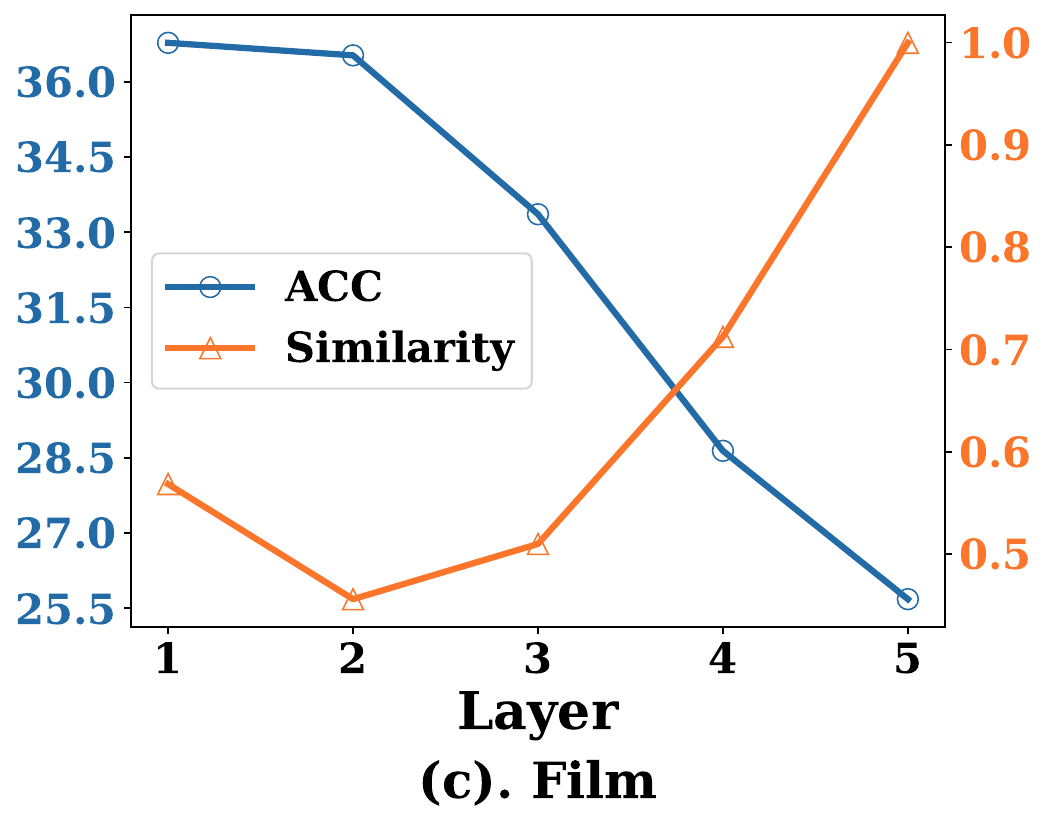}
    \includegraphics[width=0.24\textwidth]{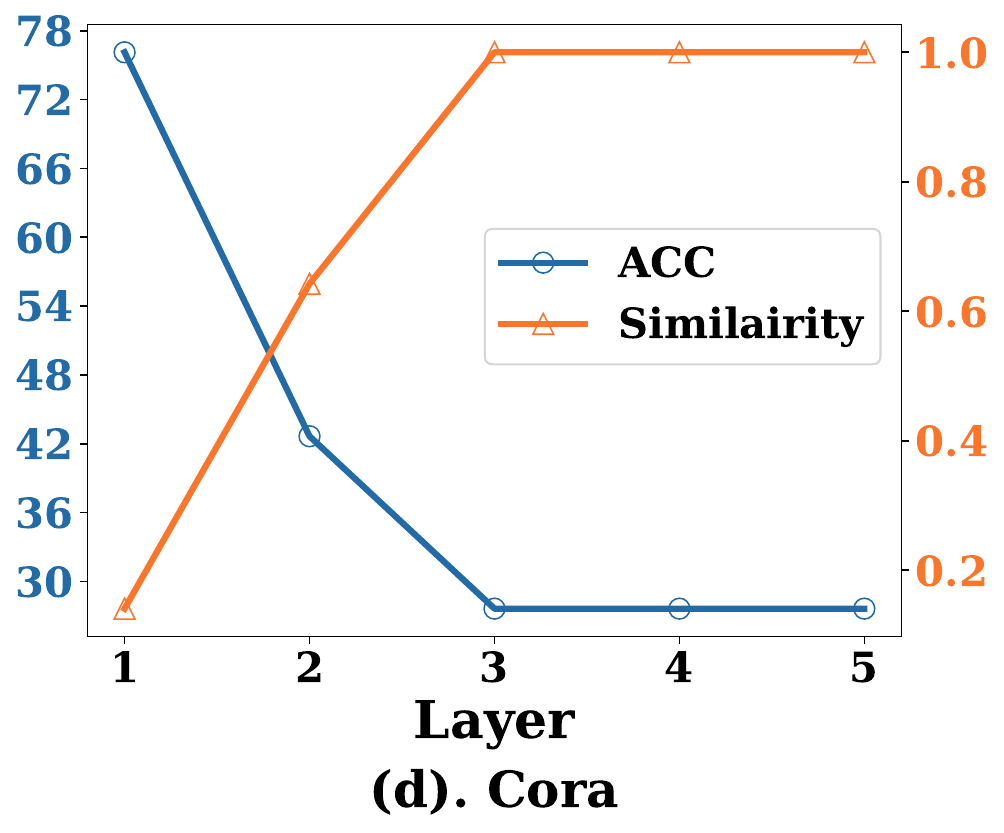}
    \vspace{-2ex}
    \caption{(a) and (b): Accuracy vs. L2 distance of final-layer Transformer representations across depths. (c) and (d): Accuracy vs. cosine similarity of final-layer representations.
    }
    \label{fig:oversmoothing}
\end{figure*}

\paragraph{\textbf{Graph Transformers.}}

Transformers have demonstrated remarkable success across diverse fields, achieving high performance in natural language processing~\cite{bert-2019} and computer vision~\cite{vit2021} applications. 
Transformers employ an encoder-decoder architecture, 

where the building block is the self-attention~\cite{Transformer} mechanism.

The self-attention module operates on a sequence of $n$ tokens with feature dimension $d$, $\mathbf{H} \in \mathbb{R}^{n\times d}$, which can be represented by a transformation function $\phi_\theta(\cdot) : \mathbb{R}^{n \times d} \rightarrow \mathbb{R}^{n \times d}$, formulated as:
\begin{align}
\mathbf{Q} &= \mathbf{H} \mathbf{W_{Q}}, \mathbf{K} = \mathbf{H} \mathbf{W_{K}}, \mathbf{V} = \mathbf{H} \mathbf{W_{V}}, \label{eq:qkv}\\
\mathbf{\hat{A}} &= \text{Softmax}\left(\frac{\mathbf{Q}\mathbf{K}^{T}}{\sqrt{d}}\right),
{\mathbf{\hat{H}}} = \mathbf{\hat{A}} \mathbf{V}, \label{eq:attention:final}
\end{align}
where the input $\mathbf{H}$ is first transformed into query, key and value matrices, $\mathbf{Q}, \mathbf{K}, \mathbf{V} \in \mathbb{R}^{d\times d}$, by linear weight matrices $\mathbf{W_{Q}}, \mathbf{W_{K}}, \mathbf{W_{V}} \in \mathbb{R}^{d\times d}$, respectively. 
Then, the attention matrix $\mathbf{\hat{A}} \in \mathbb{R}^{n \times n}$ is computed by the inner product of the query and key, followed by the normalization of a $\text{Softmax}(\cdot)$ function.
Here, $\mathbf{\hat{A}}_{ij} \in [0, 1]$ indicates the influence of $\mathbf{H}_j$ on $\mathbf{H}_i$. 
The attention matrix is finally applied to the value matrix to generate a new embedding of the tokens $\mathbf{\hat{H}} \in \mathbb{R}^{n \times d}$. In practice, Transformers employ multi-head self-attention that applies multiple parallel self-attention in different subspaces to the input and concatenates the result. 

From the perspective of GNN, the self-attention module can be viewed as a graph attention network~\cite{GAT} over a fully-connected graph. The attention matrix $\mathbf{\hat{A}}$ serves as a learned `soft' adjacency matrix of a fully connected graph, which is used to fuse global relationships in the graph. Recently, many studies have explored the possibility of incorporating edge information as a strong inductive bias in vanilla Transformers to model graph-structured data. Techniques include auxiliary GNNs, attention biases, positional embeddings, and attention masks, collectively referred to as Graph Transformers~\cite{graph-transformer-survey-architecture-22}. 
Typically, the computational and space complexity of conventional Transformers and Graph Transformers is $\mathcal{O}(n^2)$, i.e., quadratic with respect to the number of tokens/nodes.

\section{Over-smoothing in Transformers}
\label{sec:over-smoothing}
Over-smoothing refers to a phenomenon where, as the number of layers in a GNN increases, the node representations (embeddings) become indistinguishable with each other, which weakens the translatability and expressivity of GNNs~\cite{Dropedge-20}. Recently, \cite{attention-oversmoothing-bert, attention-oversmoothing-vit} reveal that the self-attention module in Transformers also demonstrate the over-smoothing property that it will act as a low-pass filter and gradually lose the high-frequency information as the number of layers increases. Formally, the smoothing rate $\lambda$ of the self-attention module can be defined as:
\begin{definition} \label{def:sa_rate}
    Let $\mathbf{\hat{A}} = \text{Softmax}(\mathbf{P})$,  $\mathbf{P} = \mathbf{Q}\mathbf{K}^T/\sqrt{n}$.  $\alpha = \max_{i,j} \lvert \mathbf{P}_{ij} \rvert$. Define $\hat{\mathbf{H}} = \mathbf{\hat{A}}\mathbf{V}$ as the output of a self-attention module, then the smoothing rate $\lambda$ of self-attention module is defined as the lower bound of high-frequency intensity ratio to the pre- and post- attention aggregation: 
    \begin{align}
        \lambda = \inf\left(\frac{\lVert \HC{{\mathbf{H}}} \rVert_F}{\lVert \HC{\hat{\mathbf{H}}} \rVert_F}\right),
    \end{align}
    where $\HC{\mathbf{H}}$ represents the complementary high-frequency component of signal $\mathbf{H}$ under the Fourier transform $\mathcal{F}$. A formal definition of $\HC{\mathbf{H}}$ can be found in \href{https://github.com/chaohaoyuan/ParaFormer/}{Appendix \ref{appendix:explain}}.  
\end{definition}

From~Definition~\ref{def:sa_rate}, a larger $\lambda$ value indicates accelerated over-smoothing. In~\cite{attention-oversmoothing-vit}, the authors have proven that the smoothing rate $\lambda$ of the vanilla self-attention module is $\sqrt{\frac{ e^{2\alpha} + n - 1}{n e^{2\alpha} \lVert\mathbf{W}_V\rVert_2}}$.

To validate the over-smoothing effects in Transformers, we conduct an empirical study of over-smoothing on a vanilla Transformer by varying depths. Specifically, we train vanilla Transformers with depths (number of layers) ranging from 1 to 5 on the Cora and Film datasets for node classification. We employ two metrics, pairwise L2-distance $\mathcal{D}_{\text{L2}}$ and cosine similarity $\mathcal{S}_{\cos}$ of all nodes.
Mathematically, these two metrics can be written as:
\begin{align}
    \mathcal{D}_{\text{L2}} &= \frac{1}{n(n-1)} \sum_{i=1}^{n} \sum_{j \neq i}^{n} \left\| \mathbf{H}_i - \mathbf{H}_j \right\|_2, \\
    \mathcal{S}_{\cos} &= \frac{1}{n(n-1)} \sum_{i=1}^{n} \sum_{j \neq i}^{n} \frac{\mathbf{H}_i \cdot \mathbf{H}_j}{\left\| \mathbf{H}_i \right\|_2 \left\| \mathbf{H}_j \right\|_2},
\end{align}
where $n$ is the number of nodes in the graph and $\mathbf{H}$ indicates the learned node representations. According to the definition, we can find lower L2-distance or higher cosine similarity indicate stronger over-smoothing effect.

Figure~\ref{fig:oversmoothing} delineates the curves of classification accuracy and the average L2-distance between the embeddings in the final layer of all pairs of nodes in the graph. As shown in Figure~\ref{fig:oversmoothing}, the strong correlation between the accuracy and the embedding distance indicates that over-smoothing is a critical factor that undermines the performance of deep Transformers.

The low-pass nature of image and text representations renders Transformers relatively immune to over-smoothing effects. However,  features in a graph may not always exhibit the low-pass property, e.g., in heterophilic graphs~\cite{zheng2024graphneuralnetworksgraphs}. Graph Transformer would suffer from performance degradation due to over-smoothing. Theoretically, we can prove that:

\begin{proposition}
\label{prop:comparsion}
Given a self-attention module $\text{SA}$ and graph attention network $\text{GAT}$ initialized with the same weight matrix, we have: $\lambda_{\text{SA}} \ge \lambda_{\text{GAT}}$.
\end{proposition}
\begin{proof}
 Given that self-attention can be conceptualized as GAT operating on a fully-connected graph, we can denote the matrix of GAT as $\Mat{P}^{\text{GAT}} = \Mat{P} \odot \Mat{A}$, then $\Mat{P}^{\text{GAT}}$ will set all $0$ to $-\inf$ to ensure $0$ after Softmax. Consequently,  $\max_{i,j} \lvert \Mat{P}_{ij} \rvert \le \max_{i,j} \lvert \Mat{P}_{ij}^{\text{GAT}} \lvert$, $\lambda$ is monotonically decreasing with respect to $\alpha$.
\end{proof}

Proposition~\ref{prop:comparsion} suggests that, compared to GAT, the self-attention module exhibits a reduced likelihood of retaining high-frequency information, indicative of a more pronounced over-smoothing effect.
Intuitively, in the message passing process, Transformers establish full connectivity between all nodes, while GAT relies on sparse adjacency matrices. This structural difference enables faster information propagation in Transformers, accelerating the over-smoothing of node representations. 
Therefore, it is imperative to design Graph Transformers with the ability to alleviate the over-smoothing  while preserving high-frequency graph information.

\section{Methodology}
In this section, we first introduce the overall architecture of ParaFormer. Then,  we propose to exploit the linear attention mechanism to optimize its time complexity from $\mathcal{O}(Kn^3)$ to $\mathcal{O}(Kn)$, where $K$ is the number of internal layers. At last, we provide the theoretical analysis of ParaFormer.

\begin{figure*}[t]
    \centering
    \includegraphics[width=0.9\textwidth]{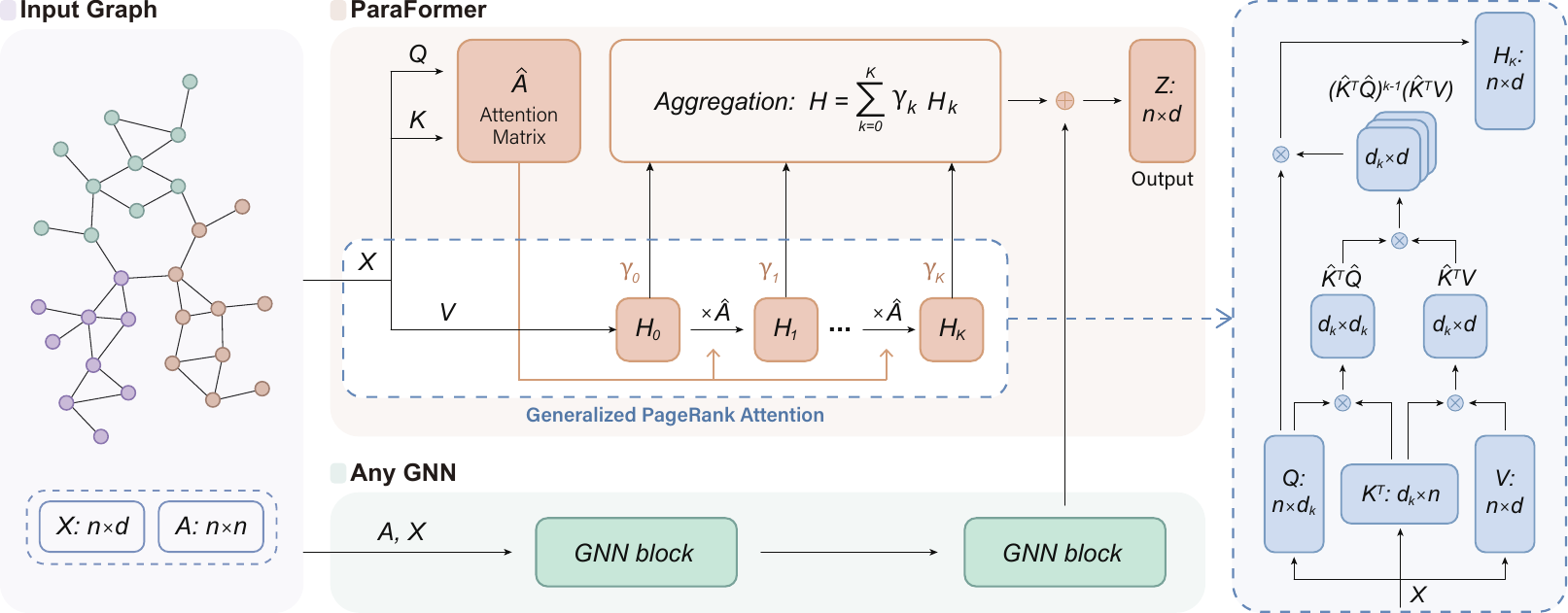}
    \caption{Overview of ParaFormer. The middle part shows the conceptual idea of ParaFormer, mainly the PageRank-enhanced attention mechanism. The right part illustrates how we implement the efficient and scalable ParaFormer, maintaining the linear complexity in detail.
    }
    \label{fig:paraformer}
\end{figure*}

\subsection{ParaFormer}
\label{sec:paraformer}
Figure~\ref{fig:paraformer} depicts the overall architecture of ParaFormer. ParaFormer contains two modules, \emph{Generalized PageRank Attention} and \emph{Auxiliary GNN Fusion}, to model the global and local information in graphs, respectively. In the following, we elucidate the design details of ParaFormer.

\subsubsection{Generalized PageRank Attention} 
\label{sec.parafomer.gpa}

In the field of graph learning, Generalized PageRank (GPR)~\cite{GPR-1, GPR-2}, enhanced from Personalized PageRank (PPR)~\cite{PageRank-98}, adaptively assigns each node a score that 
 captures its global proximity property. Precisely, the score of each node $v_i$ is initialized uniformly as $\Mat{h}^{(0)}_i = 1/n$, and the score of each node is iteratively accumulated into the current PPR score as:

 \vspace{-2ex}
 \begin{align}
    \label{eq:ppr}
     \Mat{h}^{(k)} &= (1 - \alpha) \mathbf{\Tilde{A}} \Mat{h}^{(k - 1)} + \alpha \Mat{h}^{(k-1)}, \\
     \label{eq:gpr}
     \Mat{h} &= \sum_{k = 0}^{\infty} \gamma_k \Mat{h}^{(k)} = \sum_{k=0}^{\infty} \gamma_k \mathbf{\Tilde{A}}^k \Mat{h}^{(0)},
 \end{align}
 where $\mathbf{\Tilde{A}} \in [0, 1]^{n \times n}$ is the transition probability of PPR and $\alpha \in (0, 1)$ is the damping factor. For GPR score $\Mat{h}$ in \cref{eq:ppr}, $\gamma_k \in \mathbb{R}, k = 0, 1, 2, \cdots$, referred to as the GPR weights, can be either fixed or learnable, specifying the importance of the information from length-$k$ paths. This length is usually truncated to a fixed hyper-parameter $K$ in practical computation. 
 
 By conceptualizing the attention matrix $\mathbf{\hat{A}}$ as a weighted graph, it is possible to analogize the PageRank over the attention matrix $\mathbf{\hat{A}}$. Motivated by this, we introduce a novel deep attention mechanism called Generalized PageRank Attention (GPA). 

Formally, analogous to how GPR propagates and accumulates PPR by the transition probability matrix, GPA propagates and accumulates the value vectors by a standard attention matrix $\mathbf{\hat{A}}$ as:

\begin{align}
\label{eq:gpa}
    \mathbf{Z} = \sum_{k = 0}^{K} \gamma_{k} \mathbf{\hat{A}}^k \mathbf{V} = \sum_{k = 0}^{K} \gamma_{k} \left(\text{Softmax}(\mathbf{Q}\mathbf{K}^T)\right)^k \mathbf{V}.
\end{align}
Here, $\{\gamma_k \in \mathbb{R}\mid k = 0, 1, 2, \cdots, K\}$ is a set of learnable weights of GPA, and $\mathbf{Q}$, $\mathbf{K}$ and $\mathbf{V}$ are the query, key and value matrices derived from Equation~\ref{eq:attention:final}, respectively. $\mathbf{Z} \in \mathbb{R}^{n \times d}$ is the output node representation. 
Notably, GPA can be seen as conducting $K$-layer attention with adaptive weights $\{\gamma_k \mid k = 0, 1, 2, \cdots, K\}$, and shared parameter matrices  $\mathbf{W}_Q$, $\mathbf{W}_K$, $\mathbf{W}_V$ of $K$ layers. In other words, without multiple parameter matrices, the GPA module is able to effectively capture the multi-hop dependency between nodes. Additionally, the GPA module can be proven to alleviate the over-smoothing and is capable of modeling the high frequency information in graphs, as detailed in Section \ref{sec.theory}.

\subsubsection{Auxiliary GNN Fusion}
To incorporate the local proximity information from the adjacency matrix $\mathbf{A}$, following the prior work \cite{SGFormer23}, we use a lightweight auxiliary GNN module to fuse the local proximity information with the output of the GPA module. We introduce a weight hyper-parameter $\beta$ to balance the contributions of local and global information. Therefore, the final node embedding $\mathbf{\hat{Z}} \in \mathbb{R}^{n\times d}$ is:
\begin{equation}
    \label{eq:integrate}
    \mathbf{\hat{Z}} = (1 - \beta) \mathbf{Z} + \beta \text{GNN}(\mathbf{H}, \mathbf{A}).
\end{equation}
Notably, there are many approaches integrating structural information into Transformers \cite{graph-transformer-survey-architecture-22}. ParaFormer is compatible with these approaches, enabling enhanced performance through seamless integration.

\subsubsection{Training Objective}

The output embedding  $\mathbf{\hat{Z}}$ will be adapted to different graph tasks with different training objectives.  

\paragraph{\textbf{The node classification task.}} We use an MLP $\phi_{\theta}: \mathbb{R}^{d} \rightarrow \mathbb{R}^{c}$, followed by a Softmax function to compute the probability of each class as:
\begin{align}
\label{eq:classification:mlp} 
\mathbf{\hat{P}} = \text{Softmax}(\phi_\theta(\mathbf{\hat{Z}})).
\end{align}
Given the ground truth of class labels for a node $v_i$, denoted as $\mathbf{Y_i} \in \mathbb{R}^{c}$, the model is trained by minimizing the cross-entropy loss for each node in a training set $\mathcal{D}$.
\begin{align}
    \mathcal{L} = -\frac{1}{|\mathcal{D}|}\sum_{v_i \in \mathcal{D}}  \mathbf{Y}_i \log \mathbf{\hat{P}}_i.
\end{align}

\paragraph{\textbf{The graph classification task. }} We first use a pooling function $\mu(\cdot): \mathbb{R}^{n \times d} \to \mathbb{R}^{1 \times d}$ to produce the graph-level representation $\mathbf{g} = \mu(\mathbf{\hat{Z}})$. Similar to the node classification task, given the graph representation $\mathbf{g}$,  we use a Softmax function to compute the probability of each class and the cross-entropy loss as the optimization objective. 
 
\subsection{Scalability Optimization of ParaFormer}

As discussed in Section~\ref{sec.parafomer.gpa}, the computation of GPA needs computing the power of the attention matrix with $n \times n$ dimension, leading to $\mathcal{O}(Kn^3)$ time complexity and $\mathcal{O}(n^2)$ space complexity. 
This high complexity is prohibitive for scaling to large graphs. 
Therefore, it is crucial to optimize the computation cost of GPA to enhance the scalability of ParaFormer. 

\begin{algorithm}[t]
\small
\SetAlgoLined
\SetKwRepeat{Do}{do}{while}
\KwIn{Input node feature $\mathbf{X}$, hyperparameter $K$, learnable weights $\{\gamma_0, \cdots, \gamma_{K}\}$, $\mathbf{W}_Q$, $\mathbf{W}_K$, $\mathbf{W}_V$
}
\KwOut{$\mathbf{Z}$}
$\mathbf{Q} = \mathbf{X} \mathbf{W_Q}$, $\mathbf{K} = \mathbf{X} \mathbf{W_K}$, $\mathbf{V} = \mathbf{X} \mathbf{W_V}$\\
$\mathbf{\hat{Q}} = \text{Softmax}(\mathbf{Q})$, $\mathbf{\hat{K}} = \text{Softmax}(\mathbf{K})$\\
$\mathbf{Z} = \gamma_0 \mathbf{V}$, 
$\mathbf{M} = \mathbf{\hat{K}}^{T}\mathbf{V}$\\
\For{$k \leftarrow 1\ \textbf{to}\ K$}{
    $\mathbf{Z} = \mathbf{Z} + \gamma_k \mathbf{\hat{Q}} \mathbf{M}$\\
    $\mathbf{M} = (\mathbf{\hat{K}}^{T} \mathbf{\hat{Q}}) \mathbf{M}$
    }
\caption{Forward Pass of the Scalable GPA}
\label{alg:scalGPA}
\end{algorithm}

\subsubsection{The Scalable GPA}
It is observed that the standard attention matrix can be approximated by a direct inner product of $\mathbf{\hat{Q}}$ and $\mathbf{\hat{K}}$ which are the normalization of original query matrix $\mathbf{Q}$ and key matrix $\mathbf{K}$ respectively~\cite{EfficientAttentionSoftmax-21}.

\begin{align}
   \label{eq:attention:appro}
    \mathbf{\hat{A}} &= \text{Softmax}(\mathbf{Q} \mathbf{K}^T)  \approx \mathbf{\hat{Q}} \mathbf{\hat{K}}^T ,\\
    \mathbf{\hat{Q}} &= \text{Softmax}_i(\mathbf{Q}), \mathbf{\hat{K}} = \text{Softmax}_j(\mathbf{K}).
\end{align}
This approximation can be regarded as using kernel method \cite{RFM, Performer-21} to approximate the attention matrix. The random feature transformation $\varphi: \mathbb R^d \to \mathbb R^{d'}$ can approximate the non-negative kernel function: $\kappa(\mathbf a, \mathbf b) \approx \varphi(\mathbf a)^\top\varphi(\mathbf b)$, where $\kappa(\cdot, \cdot): \mathbb R^d \times \mathbb R^d \rightarrow \mathbb R^+$ is to measure the pairwise similarity, \textit{e.g.,} the attention matrix in Transformer. Here, the transformation function $\varphi$ is the Softmax function \cite{EfficientAttentionSoftmax-21}. 
Based on the  approximation in Equation \ref{eq:attention:appro},  the power of the attention matrix can be reformulated as:

\begin{equation}
\begin{split}
    \label{our_attention}
    \mathbf{\hat{A}}^k \mathbf{V} 
    = (\mathbf{\hat{Q}}\mathbf{\hat{K}}^T)^k \mathbf{V}
    = \mathbf{\hat{Q}} (\mathbf{\hat{K}}^T \mathbf{\hat{Q}})^{k-1} (\mathbf{\hat{K}}^T \mathbf{V}).
\end{split}
\end{equation}
By applying the association rule of matrix multiplication, in Equation \ref{our_attention}, we have $\mathbf{\hat{K}}^T \mathbf{\hat{Q}} \in \mathbb{R}^{d \times d}$ and $(\mathbf{\hat{K}}^T \mathbf{V}) \in \mathbb{R}^{d \times d}$, which can be precomputed before the power iteration. In addition, in the $K$ iterations, the intermediate result  $(\mathbf{\hat{K}}^T \mathbf{\hat{Q}})^{k-1} (\mathbf{\hat{K}}^T \mathbf{V})$ in $k$-th iteration can be cached and reused in $k+1$-th iteration.  Algorithm~\ref{alg:scalGPA} presents the overall process of computing the scalable GPA, where the time complexity and space complexity are reduced to $\mathcal{O}(Kn)$ and $\mathcal{O}(nd)$, respectively.

\subsubsection{Complexity Analysis}

Suppose a graph has $n$ nodes and $m$ edges.
The complexity of computing GPA is $\mathcal{O}(Kn)$.
By default, we use a two-layer GCN for the auxiliary fusion, yielding a time complexity of $\mathcal{O}(m)$. 
In total, the overall time complexity of ParaFormer is $\mathcal{O}(Kn + m)$. 
\begin{table*}[htbp]

\setlength{\tabcolsep}{6mm}
  \centering
  \caption{The node classification results on homophilic and heterophilic graphs, with the mean accuracy and standard deviation over five runs. The highest accuracy is in  \textbf{bold} and the second highest accuracy is underlined. OOM indicates out-of-memory. }
    \resizebox{\textwidth}{!}
    {\begin{tabular}{llll|llll}
    \toprule
    \textbf{Method} & \textbf{Cora} & \textbf{CiteSeer} & \textbf{PubMed} & \textbf{Film} & \textbf{Squirrel} & \textbf{Chameleon} & \textbf{Deezer} \\
    \midrule
    MLP   & 75.7 ± 2.0 & 74.0 ± 2.0 & 87.2 ± 0.4 & 36.5 ± 0.7 & 36.6 ± 1.8 & 36.7 ± 4.7 & 66.6 ± 0.7 \\
    GCN   & 87.1 ± 1.0 & 76.5 ± 1.4 & 88.4 ± 0.5 & 30.1 ± 0.2 & 38.6 ± 1.8 & 41.3 ± 3.0 & 62.7 ± 0.7 \\
    GAT   & 88.0 ± 0.8 & 76.6 ± 1.2 & 86.3 ± 0.5 & 29.8 ± 0.6 & 35.6 ± 2.1 & 39.2 ± 3.1 & 61.7 ± 0.8 \\
    SGC   & 86.6 ± 0.3 & 76.2 ± 0.3 & 83.5 ± 0.1 & 27.0 ± 0.9 & 39.3 ± 2.3 & 39.0 ± 3.3 & 62.3 ± 0.4 \\
    JKNet & 87.0 ± 0.3 & 77.7 ± 0.4 & 87.4 ± 0.1 & 30.8 ± 0.7 & 39.4 ± 1.6 & 39.4 ± 3.8 & 61.5 ± 0.4 \\
    APPNP & 88.0 ± 0.2 & 79.3 ± 0.4 & 87.0 ± 0.2 & 31.3 ± 1.5 & 35.3 ± 1.9 & 38.4 ± 3.5 & 66.1 ± 0.6 \\
    GPRGNN & 88.5 ± 1.0 & 77.1 ± 1.8 & 87.6 ± 0.4 & 34.6 ± 1.2 & 39.0 ± 2.0 & 39.9 ± 3.3 & 66.9 ± 0.5 \\
    H$_{2}$\text{GCN} & 87.9 ± 1.2 & 77.1 ± 1.6 & 89.5 ± 0.4 & 34.4 ± 1.7 & 35.1 ± 1.2 & 38.1 ± 4.0 & 66.2 ± 0.8 \\
    SIGN  & 86.5 ± 0.6 & 76.6 ± 0.6 & 89.4 ± 0.1 & 36.5 ± 1.0 & 40.7 ± 2.5 & 41.7 ± 2.2 & 66.3 ± 0.3 \\
    CPGNN & 87.2 ± 1.1 & 75.5 ± 1.8 & 89.1 ± 0.6 & 34.5 ± 0.7 & 38.9 ± 1.2 & 40.8 ± 2.0 & 65.8 ± 0.3 \\
    GLoGNN & 88.3 ± 1.1 & 77.4 ± 1.7 & 89.6 ± 0.4 & 36.4 ± 1.6 & 35.7 ± 1.3 & 40.2 ± 3.9 & 65.8 ± 0.8 \\
    \midrule
    Graphormer & 72.9 ± 0.8 & 66.2 ± 0.8 & 82.8 ± 0.2 & 33.9 ± 1.4  & 39.9 ± 2.4  & 41.3 ± 2.8  & OOM  \\
    GraphTrans & 86.7 ± 0.6 & 72.1 ± 0.8 & 87.8 ± 0.6 & 32.1 ± 0.8 & 40.6 ± 2.4 & 42.2 ± 2.9 & OOM \\
    NodeFormer & \underline{88.7 ± 0.2} & 76.1 ± 0.2 & 89.3 ± 0.1 & 36.9 ± 1.0  & 38.5 ± 1.5  & 34.7 ± 4.1  & 66.4 ± 0.7  \\
    SGFormer & 88.1 ± 0.7 & 77.4 ± 0.8 & 89.7 ± 0.2 & 37.2 ± 1.6 & 41.8 ± 2.2  & \underline{44.9 ± 3.9}  & 67.1 ± 1.1  \\
    NAGphormer & 88.2 ± 0.5 & \underline{80.1 ± 0.2} & 89.7 ± 0.2 & 36.4 ± 1.3  & 39.8 ± 1.7  & 42.4 ± 5.7  & 63.5 ± 0.6  \\
    CobFormer-G & 88.4 ± 0.4 & 76.4 ± 0.4 & 87.3 ± 0.2 & 30.9 ± 1.7 & 35.4 ± 1.1 & 39.9 ± 2.4 & 64.2 ± 0.7 \\
    CobFormer-T & 88.2  ± 1.2 & 76.5 ± 0.7 & 89.0 ± 0.4 & 34.9 ± 3.3 & 33.1 ± 1.6 & 39.7 ± 4.6 & 66.6 ± 0.7 \\
    Polynormer & 87.9 ± 1.2 & 77.4 ± 0.7 & \textbf{90.3 ± 0.1} & 36.8 ± 1.5 & 40.9 ± 2.0  & 41.8 ± 3.5  & \underline{67.6 ± 0.4} \\
    \midrule
    \textbf{ParaFormer$_\text{GCN}$} & \underline{88.7 ± 0.5} & 78.6 ± 0.1 & \textbf{90.3 ± 0.2} & \textbf{38.0 ± 1.2} & \textbf{43.0 ± 2.8} & \textbf{45.9 ± 3.3} & 67.3 ± 0.9 \\
    \textbf{ParaFormer$_\text{GPRGNN}$} & \textbf{89.4 ± 0.9} & \textbf{80.5 ± 0.6} & \underline{90.0 ± 0.1} & \underline{37.3 ± 1.0} & \underline{42.1 ± 2.1} & 42.1 ± 4.5 & \textbf{67.7 ± 0.5} \\
    \bottomrule
    \end{tabular}
    }
  \label{tab:medium-graph-results}
  
\end{table*}
\subsection{Theoretical Analysis}\label{sec.theory}
This section theoretically analyzes the frequency property and over-smoothing of ParaFormer. 

\subsubsection{ParaFormer is an adaptive-pass filter.}
While the vanilla Transformer serves as a pure low-pass filter~\cite{attention-oversmoothing-vit}, with the help of the adaptive weights, ParaFormer can flexibly process graphs with both low- and high-frequency components.

\begin{theorem}
Given $\gamma_1 > 0 , \gamma_k=(-a)^k/2, a\in(0,\frac{1}{n}), k>1$, the attention matrix $\hat{\Mat{A}}$ in ParaFormer is a polynomial graph filter that can function as both a low-pass and high-pass graph filter.
    
    \label{proposition:filter}
\end{theorem}
The detailed proof is in \href{https://anonymous.4open.science/r/ParaFormer-WSDM/}{Appendix \ref{appendix:proof1}}. As proven in~\cite{attention-oversmoothing-vit}, vanilla Transformers will suppress the high frequency components, leading to Transformers' inferior performance. On the contrary, when the value of $\gamma_k$ can be negative, ParaFormer can also process information with high-frequencies. As justified in Section~\ref{sec.impactofk}, the conditions on $\gamma$ is relatively easy to satisfy in real optimizations, leading the better performance.

\subsubsection{Over-smoothing Analysis. }Regarding the over-smoothing properties of ParaFormer, we have the following results.

\begin{theorem}
     \label{theorem:oversmoothingcompare}
     Given a Generalized PageRank Attention $\text{GPA}$ with proper initialization of $\{\gamma_k\}$ and a self-attention module $\text{SA}$, we have:
     $\lambda_{\text{GPA}} \le \lambda_{\text{SA}}$.
\end{theorem}

As shown in Theorem~\ref{theorem:oversmoothingcompare}, even with $K$-layers, the GPA module exhibits slower smoothing rate than that of single-layer self-attention. This indicates that ParaFormer not only effectively captures multi-hop dependencies but also mitigates the impact of over-smoothing, thereby achieving superior performance.

Theorem~\ref{theorem:oversmoothingcompare} analyzes the situation when the weight of each layer is fixed.  However, with adaptive weights,  ParaFormer can automatically regulate over-smoothing under an appropriate optimization objective. Specifically, we have:
\begin{theorem}
    \label{proposition:oversmoothing}
    Suppose $K$ is sufficiently large, for PageRank-enhanced attention, $\hat{\mathbf{A}}^{k}\mathbf{V}, \forall k\geq k'$, will be over-smoothed. When over-smoothing happens in ParaFormer, the learnable $\gamma_k$ will converge to $0$ with an appropriate learning rate.
\end{theorem}

The proof can be found in \href{https://anonymous.4open.science/r/ParaFormer-WSDM/}{Appendix \ref{appendix:proof3}}. Intuitively, Theorem \ref{proposition:oversmoothing} can be understood as, when the over-smoothed $\mathbf{H}_{k}$ fails to classify the nodes, it will increase the training loss. Then our learnable $\gamma_k$ will reduce its magnitude to reduce the training loss by gradient descent, which also alleviates the over-smoothing effect. 
Since $|\gamma_{k}|$ will converge to $0$ with $\mathbf{H}_{k}$ having been over-smoothing, the impact of over-smoothed deep layers will diminish. Thus, the $\mathbf{H}_{k}$ carries useful information to distinguish node representations will dominate the $\mathbf{H}$. Additionally, we find empirical evidence validating the efficacy of ParaFormer in overcoming over-smoothing our experiments in Section \ref{sec.antooversm}.

\section{Experiments}
\label{sec:experiments}

In this section, we report
our comprehensive experiments in the following facets:
\ding{172} Compare the effectiveness of ParaFormer with SOTA for node classification on homophilic and
heterophilic graphs.
\ding{173} Test the effectiveness of ParaFormer for graph-level classification task.
\ding{174} Evaluate the scalability of ParaFormer on graphs with millions of nodes. 
\ding{175} Conduct ablation studies on the core components of ParaFormer. 
\ding{176} Investigate the capability of ParaFormer in  alleviating over-smoothing.

\paragraph{\textbf{Implementation Details.}} All experiments involving ParaFormer in this study are conducted on Nvidia HGX H20 Enterprise 96GB. The primary software libraries utilized are Python 3.10, Pytorch 2.0.1 and PyG 2.5.2. 

The main features of the datasets we utilize in this work are listed in Appendix~\ref{appendix.dataset}. For datasets with predefined splits provided by prior work, these splits are adopted, and an equivalent number of experimental repetitions are performed. For all other datasets, five experimental runs are conducted, and the results are reported as mean and variance. Following the conventions of prior work, accuracy is reported as the evaluation metric for all experiments.

\subsection{Node Classification}
\label{sec: Medium-sized Graph}

\paragraph{\textbf{Setup.}} We use seven node classification benchmarks, including two categories, homophilic (Cora, Citeseer and PubMed~\cite{Sen08collectiveclassification}) and heterophilic (Film, Squirrel, Chameleon and Deezer~\cite{rozemberczki2021multiscale}),  based on whether neighboring nodes tend to share the same or distinct labels, to evaluate the performance of ParaFormer. We split the data of homophilic graphs following the traditional fully-supervised setting~\cite{he2021bernnet}. For heterophilic graphs, we follow the split from~\cite{newbench} (Film and Squirrel) and~\cite{makingprogress} (Squirrel and Chameleon).  

We implement two variants of ParaFormer with two different auxiliary GNN modules: GCN~\cite{GCN17} and GPRGNN~\cite{GPRGNN21}, referred to as ParaFormer$_{\text{GCN}}$ and ParaFormer$_{\text{GPRGNN}}$ respectively. More details of the experiment can refer to \href{https://anonymous.4open.science/r/ParaFormer-WSDM/}{Appendix ~\ref{appendix:imple_details}}.

\paragraph{\textbf{Baselines. }}  To extensively validate the  effectiveness of ParaFormer, we compare it with 16 competitive baselines, including message-passing GNNs and Graph Transformers. For GNNs, we use GCN~\cite{GCN17}, GAT~\cite{GAT} and SGC~\cite{SGC-19}, JKNet~\cite{JKNet-18}, APPNP~\cite{APPNP18}, GPRGNN~\cite{GPRGNN21}, H$_2$GCN~\cite{h2gcn-20}, SIGN~\cite{sign}, CPGNN \cite{CPGNN}, and GLoGNN~\cite{GLOGNN} as baselines. For Graph Transformers, we adopt Graphormer~\cite{Graphformer21}, GraphTrans~\cite{GraphTrans21}, NodeFormer~\cite{Nodeformer22}, SGFormer~\cite{SGFormer23}, NAGphormer~\cite{chen2023nagphormer}, CobFormer~\cite{cobformer}\footnote{CobFormer-G/T denotes the task head module of CobFormer, which integrates GNNs and Transformer respectively.}, and Polynormer~\cite{deng2024polynormer} as baselines. These baselines include advanced GNNs being able to capture multi-hop relations (e.g. APPNP, SIGN), mitigate over-smoothing (e.g., JKNet, GPRGNN), encode heterophily information (e.g., H$_2$GCN, CPGNN) and capture global information (e.g., GLoGNN). Additionally, we also include recent powerful Graph Transformers while neglecting mitigating over-smoothing as our baselines.

\paragraph{\textbf{Results.}}
Table~\ref{tab:medium-graph-results} presents the mean accuracy and standard deviation results on all graph datasets for all models. From Table~\ref{tab:medium-graph-results} we can observe that: 
\ding{172} ParaFormer surpasses all baselines on all datasets, indicating its strong generalization ability. 
\ding{173} ParaFormer achieves a significant 3.4\% average relative performance gain over message-passing GNNs, indicating the benefits of global attention module in ParaFormer. 
\ding{174} ParaFormer also shows a better performance compared with the other Graph Transformers (1.9\% average relative improvement), verifying the effectiveness of the proposed PageRank attention. Notably, with the same auxiliary GNN module, ParaFormer$_{\text{GCN}}$ consistently outperforms SGFormer in all datasets. It indicates that PageRank attention mechanism can enhance the generalization capacity by effectively capturing the long-range dependency, compared with one-layer attention in SGFormer. (4) ParaFormer gains a larger relative improvement (2.03\%) on heterophilic graph  than that of homophilic graph (1.00\%), implying that ParaFormer does not rely on homophily assumption in message-passing GNNs. 

\paragraph{\textbf{Efficiency and Memory Usage}}
We record the training time and memory in inference time in Table~\ref{tab:train_time_inference}, which indicates that S-GPA greatly improves the training efficiency and save the GPU usage.
\begin{table}[t]
  \centering
  \caption{Training time each epoch and memory usage during inference}
  
    \resizebox{0.48\textwidth}{!}{\begin{tabular}{llllll}
    \toprule
          & \textbf{Cora} & \textbf{PubMed} & \textbf{Chameleon} & \textbf{Squirrel} & \textbf{Deezer} \\
    \midrule

    Time(w/ S-GPA)  & 10.2ms & 93.7ms & 11.4ms & 13.0ms & 242.1ms \\
    Time(w/o S-GPA)  & 8.7ms & 385.3ms & 9.3ms & 12.3ms & 834.8ms \\
    Memory(w/ S-GPA) & 0.59G & 0.95G & 0.51G & 0.59G & 4.35G \\
    Memory(w/o S-GPA) & 0.59G & 4.21G & 0.53G & 0.64G & 11.2G \\
    \bottomrule
    \end{tabular}}
  \label{tab:train_time_inference}
\end{table}

\subsection{Graph-level Classification}
\paragraph{\textbf{Setup.}} We further test the ability of ParaFormer in classification tasks such as image and text. Specifically, we follow the setting as in DIFFormer~\cite{wu2023difformer}, choosing 20News-Group~\cite{pedregosa2011scikit} and STL-10 as the dataset with limited label rates. Wu et al.~\cite{wu2023difformer} pre-process the dataset with K-NN graph construction. The objective of these datasets is to predict the labels of images and documents, which constitutes a graph-level classification task. 
Since these datasets lack inherent graph structures, GNN baselines construct graphs using the k-nearest neighbors method, with k set to 5. To assess the GPA module's capability in capturing global properties, ParaFormer is designed without auxiliary GNNs.

\paragraph{\textbf{Baselines.}}
We compare ParaFormer with a variety of models across different aspects, including two classic graph self-supervised learning model, LP~\cite{zhu2003semi} and  ManiReg~\cite{belkin2006manifold}, classic GNNs such as GCN~\cite{GCN17}, GAT~\cite{GAT}, a graph structural learning model, GLCN~\cite{jiang2019semi}, and a Graph Transformer, DIFFormer~\cite{wu2023difformer}.

\begin{table}[t]
    \centering
    \centering
    \caption{The results on image (STL-10) and text (20-News) classification tasks.}
    \vspace{-2ex}
    \resizebox{0.48\textwidth}{!}{\begin{tabular}{llrrlll}
    \toprule
    \multicolumn{1}{c}{\multirow{2}[4]{*}{\textbf{Dataset}}} & \multicolumn{3}{c}{\textbf{ STL-10}} & \multicolumn{3}{c}{\textbf{20News}} \\
\cmidrule{2-7}          & 100 labels & \multicolumn{1}{l}{500 labels} & \multicolumn{1}{l}{1000 labels} & 1000 labels & 2000 labels & 4000 labels \\
    \midrule
    LP    & \multicolumn{1}{r}{65.2} & 71.8  & 72.7  & \multicolumn{1}{r}{55.9} & \multicolumn{1}{r}{57.6} & \multicolumn{1}{r}{59.5} \\
    ManiReg & $66.5 \pm 1.9$ & \multicolumn{1}{l}{$72.5 \pm 0.5$} & \multicolumn{1}{l}{$74.2 \pm 0.5$} & $56.3 \pm 1.2$ & $60.0 \pm 0.8$ & $63.6 \pm 0.7$ \\
    GCN-kNN & $66.9 \pm 0.5$ & \multicolumn{1}{l}{$72.1 \pm 0.8$} & \multicolumn{1}{l}{$73.7 \pm 0.4$} & $56.1 \pm 0.6$ & $60.6 \pm 1.3$ & $64.3 \pm 1.0$ \\
    GAT-kNN & $66.5 \pm 0.8$ & \multicolumn{1}{l}{$72.0 \pm 0.8$} & \multicolumn{1}{l}{$73.9 \pm 0.6$} & $55.2 \pm 0.8$ & $59.1 \pm 2.2$ & $62.9 \pm 0.7$ \\
    DenseGAT & OOM   & \multicolumn{1}{l}{OOM} & \multicolumn{1}{l}{OOM} & $54.6 \pm 0.2$ & $59.3 \pm 1.4$ & $62.4 \pm 1.0$ \\
    GLCN  & $66.4 \pm 0.8$ & \multicolumn{1}{l}{$72.4 \pm 1.3$} & \multicolumn{1}{l}{$74.3 \pm 0.7$} & $56.2 \pm 0.8$ & $60.2 \pm 0.7$ & $64.1 \pm 0.8$ \\
    DIFFomer-s & $67.8 \pm 1.1$ & \multicolumn{1}{l}{$73.7 \pm 0.6$} & \multicolumn{1}{l}{$76.4 \pm 0.5$} & $57.7 \pm 0.3$ & $61.2 \pm 0.6$ & $65.9 \pm 0.8$ \\
    DIFFomer-a & $66.8 \pm 1.1$ & \multicolumn{1}{l}{$72.9 \pm 0.7$} & \multicolumn{1}{l}{$75.3 \pm 0.6$} & $57.9 \pm 0.7$ & $61.3 \pm 1.0$ & $64.8 \pm 1.0$ \\
    \midrule
    \textbf{ParaFormer} & $\mathbf{68.5 \pm 1.3}$ & $\mathbf{74.3 \pm 0.2}$      &  $\mathbf{76.5 \pm 0.3}$     & $\mathbf{59.5 \pm 0.4}$ & $\mathbf{62.8 \pm 0.4}$ & $\mathbf{67.2 \pm 1.0}$ \\
    \bottomrule
    \label{tab:classification}
    \end{tabular}}

    \vspace{-1ex}
\end{table}

\paragraph{\textbf{Results.}} Table~\ref{tab:classification} exhibits the accuracy across the two datasets with different numbers of labels. As shown in Table~\ref{tab:classification}, ParaFormer consistently outperforms all baselines across all settings, indicating the effectiveness of the generalized PageRank attention. Notably, ParaFormer has higher relative improvement, when labels are more limited, indicating its strong generalization ability.

\subsection{Large Graph Results}
\label{sec: Large-sized Graph}
\paragraph{\textbf{Setup. }}
To assess the scalability of ParaFormer, we perform node classification experiments on four representative large graph datasets with millions of nodes and edges from the open graph benchmark (OGB)~\cite{ogb-20}. These datasets include two homophilic graphs-- ogbn-arxiv  and Amazon2M, as well as two heterophilic graphs -- pokec and arxiv-year.  For Amazon2M and ogbn-arxiv, we follow the splitting protocol from~\cite{Nodeformer22} and~\cite{SGFormer23}, respectively. For the two heterophilic graphs, we follow the official splits of OGB~\cite{ogb-20}.
More details can refer to \href{https://anonymous.4open.science/r/ParaFormer-WSDM/}{Appendix ~\ref{appendix:imple_details}}.

\paragraph{\textbf{Baselines. }}
Since not all baselines in Section~\ref{sec: Medium-sized Graph} can effectively learn the representations for large-scale graphs, we select MLP, three scalable GNNs (GCN~\cite{GCN17}, SGC~\cite{SGC-19}, SIGN~\cite{sign}), and two scalable graph Transformers (NodeFormer~\cite{Nodeformer22}, SGFormer~\cite{SGFormer23}) as baselines. Additionally, we further compare two message-passing GNNs with neighbor sampling optimization~\cite{graphsaint}, dubbed GCN-Nsampler and GAT-Nsampler.

\begin{table}[t]
  \centering
  \caption{The results on large-scale graphs for node classification tasks.}
  
    \resizebox{0.48\textwidth}{!}{
    \begin{tabular}{lllll}
    \hline
    \textbf{Dataset} & \textbf{arXiv-year} & \textbf{Amazon2M} & \textbf{pokec} & \textbf{ogbn-arxiv} \bigstrut\\

    \hline
    
    GCN   & 40.58 ± 0.39 & 83.90 ± 0.10  & 62.31 ± 1.13 & 71.74 ± 0.29 \\
    SGC   & 32.83 ± 0.13 & 81.21 ± 0.12 & 52.03 ± 0.84 & 67.79 ± 0.27 \\
    GCN-Nsampler & 39.34 ± 0.31 & 83.84 ± 0.42 & 63.75 ± 0.77 & 68.50 ± 0.23 \\
    GAT-NSampler  & 36.90 ± 0.15 & 85.17 ± 0.32  & 62.32 ± 0.65 & 67.63 ± 0.23 \\
    SIGN  & 44.49 ± 0.14 & 80.98 ± 0.31 & 68.01 ± 0.25 & 70.28 ± 0.25 \bigstrut[b]\\
    \hline
    NodeFormer  & 36.50 ± 0.24 & 87.85 ± 0.24 & 70.32 ± 0.45 & 59.90 ± 0.42 \bigstrut[t]\\
    SGFormer  & 48.99 ± 0.21 & 89.09 ± 0.10 & 73.76 ± 0.24 & 72.63 ± 0.13 \bigstrut[b]\\
    \hline
    \textbf{ParaFormer} & \textbf{49.31 ± 0.40} & \textbf{89.36 ± 0.14} & \textbf{74.93 ± 0.40} & \textbf{72.87 ± 0.21} \bigstrut\\
    \hline
    \end{tabular}
    }
  
  \label{tab:large-graph-results}
\end{table}

\begin{table*}[t]
  \setlength{\tabcolsep}{5mm}
  \centering
  \caption{The test accuracy (Acc. \%) and memory consumption (Mem. GB) of ParaFormer and ablated components. `$\textbf{L}_{\gamma}$' represents whether $\gamma_k$ is learnable. `S-GPA' indicates the scalable GPA.}
  \vspace{-2ex}
    \resizebox{\textwidth}{!}{\begin{tabular}{c|ccc|r|r|r|r|r|r|r|r|r|r}
    \toprule
    \multirow{2}{*}{\textbf{Variant}} & \multirow{2}{*}{\textbf{GPA}}    & \multirow{2}{*}{$\textbf{L}_{\gamma}$}  & \multirow{2}{*}{\textbf{S-GPA}}    & \multicolumn{2}{c}{Cora} & \multicolumn{2}{c}{PubMed} & \multicolumn{2}{c}{Chameleon} & \multicolumn{2}{c}{Squirrel} & \multicolumn{2}{c}{Deezer} \\
      & & & & {\textbf{Acc.}} & \textbf{Mem.} & {\textbf{Acc.}} & \textbf{Mem.}  & {\textbf{Acc.}} & \textbf{Mem.} & {\textbf{Acc.}} & \textbf{Mem.}  & {\textbf{Acc.}} & \textbf{Mem.} \\
    \midrule
       (a)  & \cmark & \cmark &  \cmark & \textbf{88.7}  & 0.63 & \textbf{90.3} & 1.54 & \textbf{45.9} & 0.61 & \textbf{43.3} & 0.69 & \textbf{67.3} & 4.82\\
       (b)  & \cmark & \ccross &  \cmark & 88.3  & 0.63 & 90.0 & 1.34 & 44.9 & 0.61 &42.7 & 0.69 & 66.7 & 4.68 \\
       (c)  & \cmark & \cmark &  \ccross & 88.4  & 0.74 & 90.2& 8.49 & 45.3 & 0.62 & 42.1 & 0.74 & 65.8 & 20.06\\
       (d)  & \ccross & NA &  \cmark & 87.5  & 0.63 & 89.8 & 1.10 & 44.8 & 0.61 & 42.3  & 0.69 & 65.0 & 4.52\\
       (e)   & \ccross & NA & \ccross & 88.0  & 0.74 & 89.8 & 8.49 & 45.0 & 0.61 & 42.1 & 0.73 & 63.7 & 20.06\\
   \bottomrule
    \end{tabular}
    }
    
  \label{tab:ablation study}
\end{table*}

\paragraph{\textbf{Results. }}
Table \ref{tab:large-graph-results} demonstrates consistent performance improvements across all benchmarks. The results reveal that advanced Graph Transformers, specifically SGFormer and ParaFormer, significantly outperform scalable GNNs, emphasizing the importance of all-pairs attention. The global attention mechanism supplies the model with abundant additional information beyond the local neighbors, thereby facilitating more accurate node classification. In comparison to SGFormer, which employs a basic linear attention, the improved results indicate that GPR can effectively enhance the expressiveness of the attention mechanism without necessitating a substantial increase in the number of parameters.

\begin{table}[t]

  \centering
  
  \caption{GPU memory usage and time for 5 epochs of ParaFormer and ParaFormer w/o scalable GPA.}
  
    \resizebox{0.48\textwidth}{!}{
    \begin{tabular}{lrrrrrr}
    \toprule
    \textbf{Dataset} & \textbf{Cora} & \textbf{PubMed} & \textbf{Deezer} & \textbf{ogbn-arxiv} & \textbf{pokec} & \textbf{Amazon2M} \\
    \midrule
    \# nodes & 2,708 & 19,717 & 28,281 & 169,343 & 1,632,803 & 2,440,029 \\
    \# edges & 5,278 & 44,324 & 92,752 & 1,166,243 & 30,622,564 & 61,859,140 \\
    \midrule
    Memory (w/ S-GPA) & 0.63G & 1.54G & 4.60G & 8.38G & 77.5G & 89.3G \\
    Memory (w/o S-GPA)  & 1.18G & 9.11G  & 20.83G & OOM & OOM & OOM \\
    Time (w/ S-GPA) & 0.29s & 0.50s & 1.07s & 2.69s & 23.9s & 36.08s \\
    Time (w/o S-GPA)  & 0.46s & 1.82s  & 3.85s & NA & NA & NA \\

    \bottomrule
    \end{tabular}}
    \label{tab:GPA_efficiency}
    
\end{table}

{Table~\ref{tab:GPA_efficiency} compares the training time and GPU memory consumption in 5 epochs of ParaFormer w/ and w/o the optimization of scalable GPA (S-GPA). The scalable GPA reduces the complexity to linear regarding the number of nodes and edges. ParaFormer without the scalable GPA spends extra 1-3$\times$ time and storage resources on small graphs and runs out-of-memory on large graphs. These results verify the efficacy of S-GPA in enhancing computational speed and optimizing memory usage. Please note in Table~\ref{tab:GPA_efficiency}, we input the entire graph into ParaFormer to prove linearity. Practically, extremely large graphs are trained using mini-batches~\cite{SGFormer23}, which do not require such extensive memory.
}

\subsection{Ablation Studies}

\subsubsection{The Effectiveness of Each Component in ParaFormer}

To study each component, we conduct ablation studies that use fixed or learnable weights $\{\gamma_k \mid k=0, \cdots K\}$, and removing the scalable optimization of linear attention. The default auxiliary GNN in ParaFormer is GCN. Five model variants, Variant (a)-(e) are shown in Table~\ref{tab:ablation study}. The experiments are node classification task on five datasets, where the prediction accuracy and GPU memory usage (GB) can be found in the Table~\ref{tab:ablation study}.

From the ablated experiments, we summarize five key observations: 
\ding{172} The intact ParaFormer (Variant (a)) consistently performs the best, underscoring the necessity of each component. 

\ding{173} When replacing GPA with vanilla attention, ParaFormer with vanilla attention (Variant (d)) and with linear attention (Variant (e)) manifest similar prediction accuracy. This phenomenon suggests the efficient linear attention mechanism almost does not lose information in our scenarios. 
\ding{174} If we fix the parameter $\gamma_k$, the accuracy declines. This is because a learnable $\gamma_k$ can further mitigate the over-smoothing issue as discussed in Section~\ref{sec:paraformer}. 
\ding{175} Beyond accuracy, the efficient attention mechanism (Variant (a)) significantly reduces the required GPU memory, especially on graphs with more than 10,000 nodes. ParaFormer saves $81.86\%$ and $75.97\%$ GPU memory on PubMed and Deezer datasets, respectively. 

\ding{176} Regarding GPU memory consumption, GPA and the usage of learnable weights $\gamma_k$ (Variant (a)) introduce only a slight memory cost, showcasing the real utility of ParaFormer.

\begin{figure}[t]
    \centering
    \captionsetup{aboveskip=2pt}
    \subfigure[Different $K$ in ParaFormer]{
        \includegraphics[width=0.47\columnwidth]{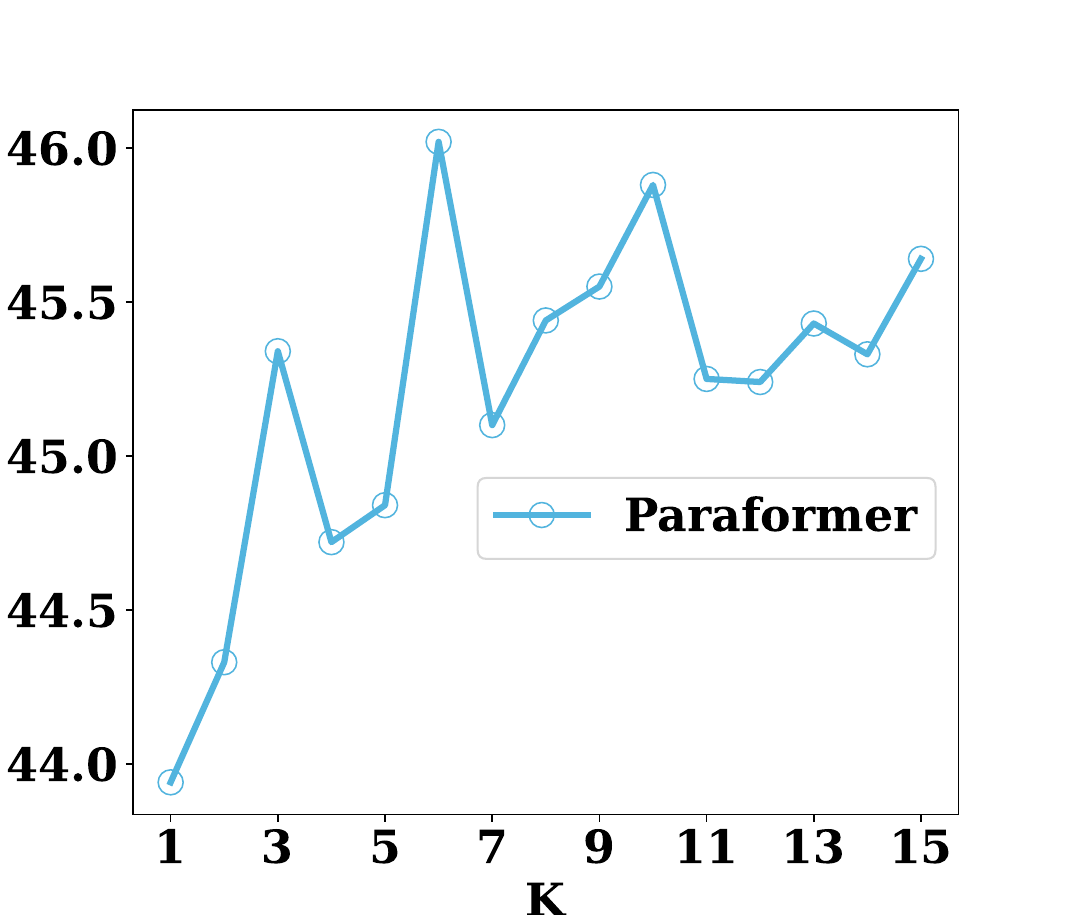}
    }
    \subfigure[Distribution of $\gamma_k$]{
        \includegraphics[width=0.47\columnwidth]{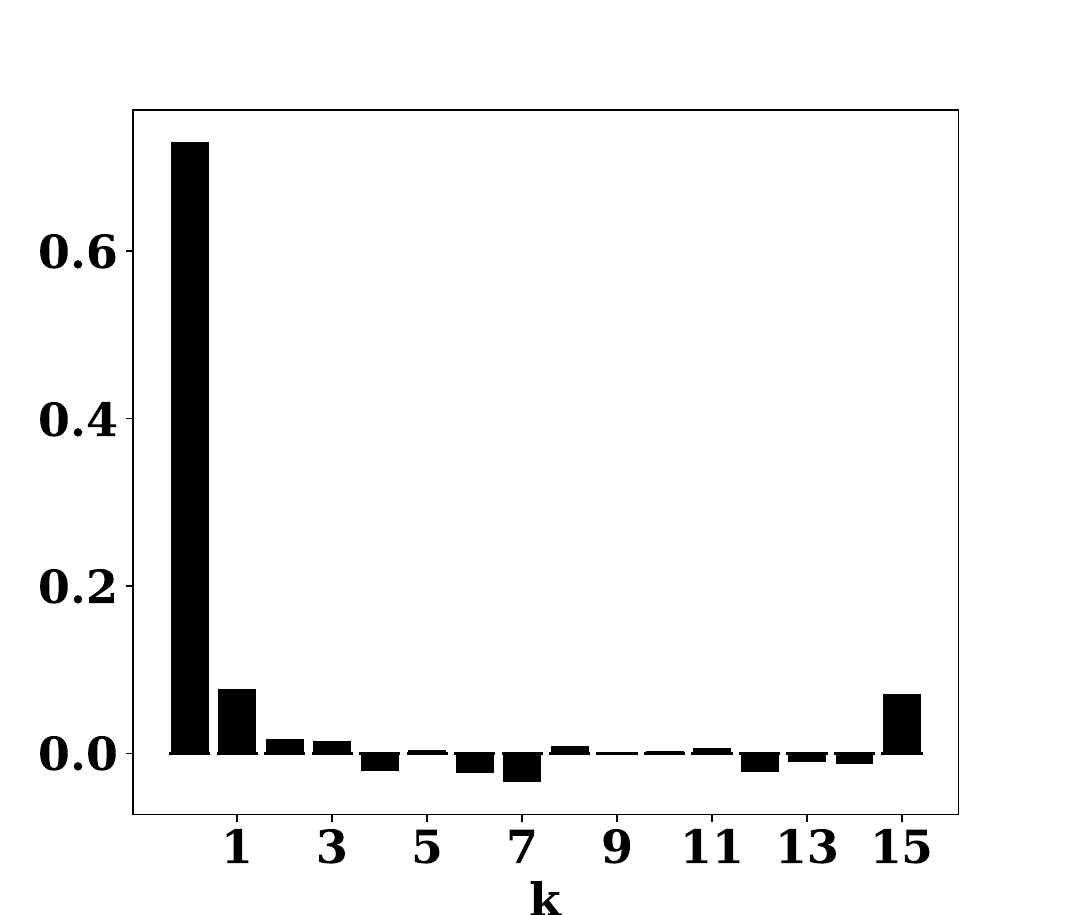}
    }
    \caption{Figure (a) shows, in Chameleon dataset, the influence of hyper-parameter K on accuracy. Figure (b) exhibits the distribution of $\gamma$ when $K=15$.
    }
    \label{fig:K-analysis}
\end{figure}

\subsubsection{Impact of the number of layers $K$}\label{sec.impactofk}
To examine the impact of the number of layers $K$, we conduct the experiments on Chameleon, varying the layer number  $K$ from 1 to 15. Figure~\ref{fig:K-analysis}(a) reports the performance curve of different $K$.   As shown in Figure~\ref{fig:K-analysis}(a), the accuracy gradually improves as the value of $K$ increases, highlighting the effectiveness of multi-hop design. Once $K$ reaches a certain size, the accuracy stabilizes. Furthermore, to investigate the behavior of adaptive weights $\gamma$, we visualize the adaptive weights of the model with $K=15$ in Figure~\ref{fig:K-analysis}(b). From Figure~\ref{fig:K-analysis}(b), we can observe: \ding{172} The value of $\gamma_k$ becomes relatively small for large values of $K$, indicating that ParaFormer can learn the necessary hop depth from the data and is not overly sensitive to the specific value of $K$. \ding{173} Certain weights of $\gamma$ exhibit negative values, consistent with the theoretical findings in \cref{proposition:filter}. It demonstrates that ParaFormer can act as a high-pass filter during the learning of heterophilic graphs.

\subsubsection{Impact of Combining Local and Global Information}
The weight hyperparameter $\mathbf{\beta}$ serves to balance the importance of local and global information within the graph. As the significance of local v.s. global information may vary between specific graphs, as illustrated in Table~\ref{tab:impact_beta}, there is no universally optimal $\mathbf{\beta}$ for all graphs.

\begin{table}[t]
  \setlength{\tabcolsep}{3mm}
  \centering
  \caption{The relation between final test accuracy and the value of $\mathbf{\beta}$.}
  
    \begin{tabular}{lrrrrr}
    \toprule
    \textbf{Value of} $\mathbf{\beta}$ & \textbf{0} & \textbf{0.3} & \textbf{0.5} & \textbf{0.7} & \textbf{1} \\
    \midrule
    Cora  & 77.4  & 88.1  & 87.4  & 88.7  & 87.6 \\
    PubMed & 88.3  & 90.2  & 90.1  & 90.3  & 89.6 \\
    Chameleon & 27.6  & 41.8  & 42.6  & 45.9  & 43.0 \\
    \bottomrule
    \end{tabular}
  \label{tab:impact_beta}
\end{table}

\begin{figure*}[t]
    \centering
    \captionsetup{aboveskip=2pt}
    \subfigure[Cora]
    {\includegraphics[width=0.24\textwidth]{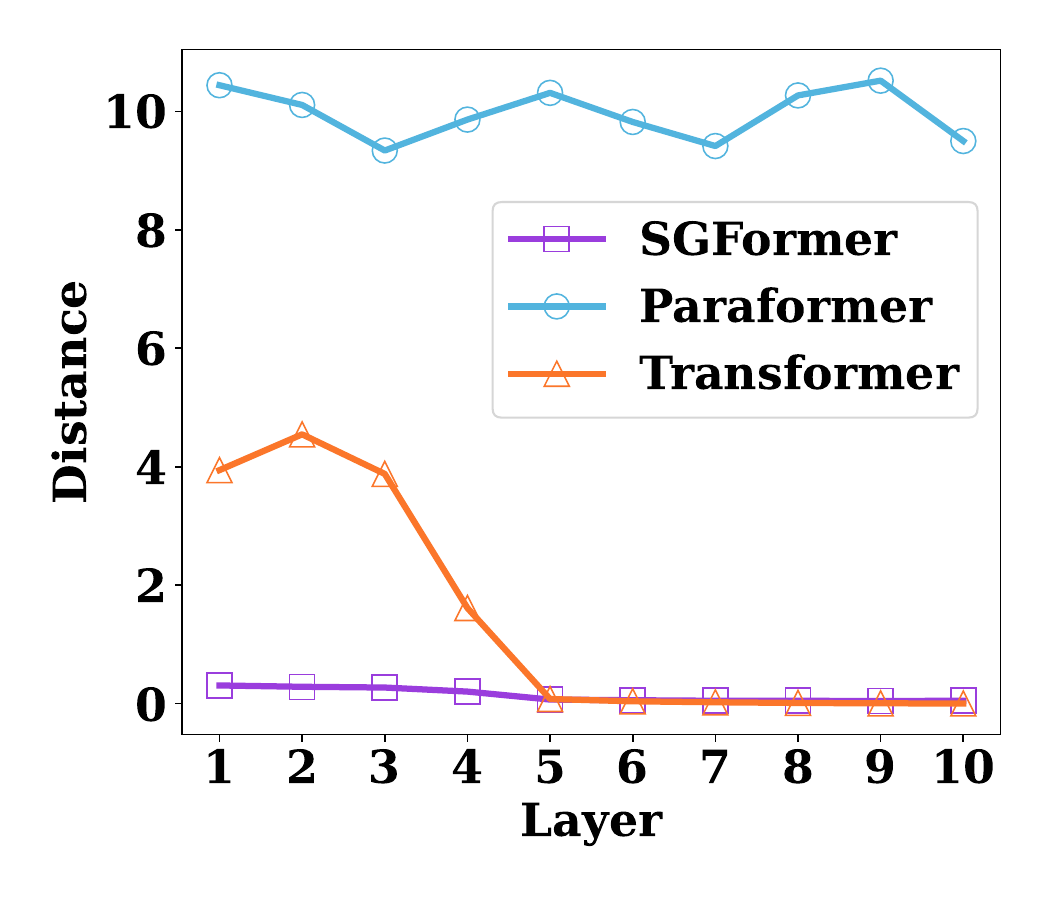}}
    \subfigure[Film]
    {\includegraphics[width=0.24\textwidth]{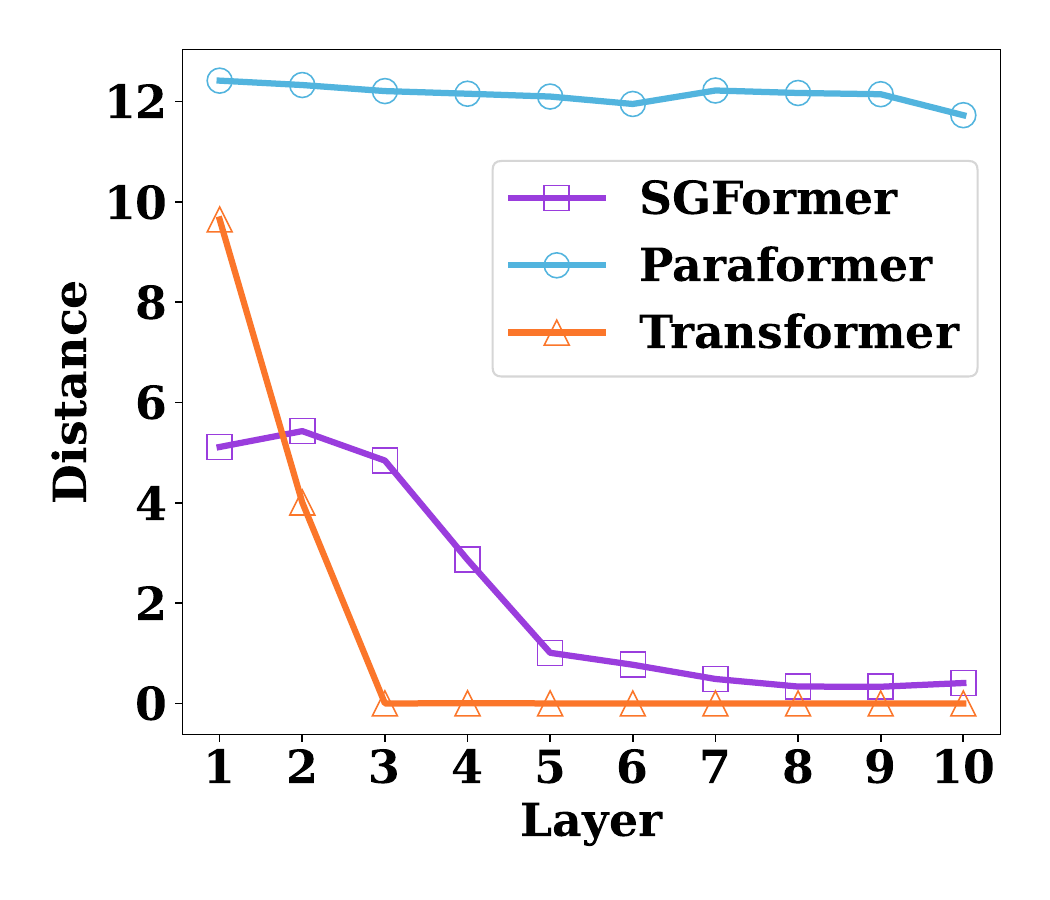}}
    \subfigure[Cora]
    {\includegraphics[width=0.24\textwidth]{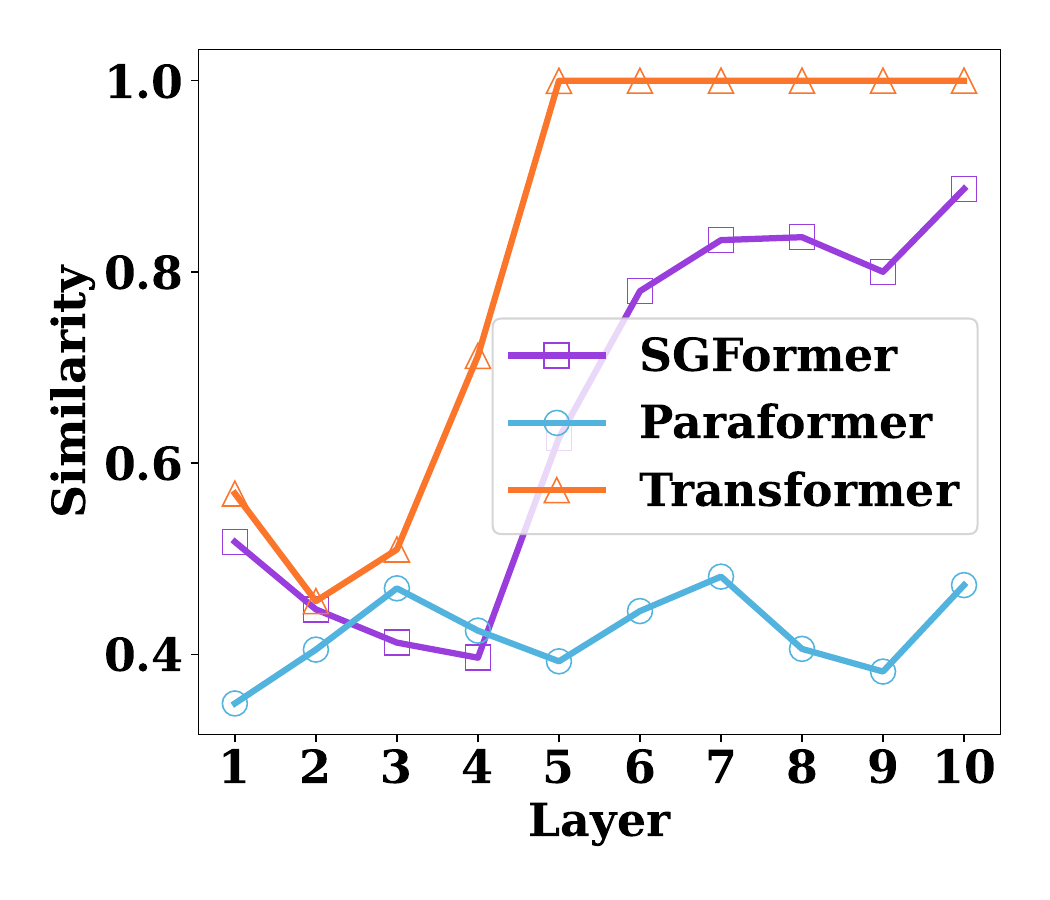}}
    \subfigure[Film]
    {\includegraphics[width=0.24\textwidth]{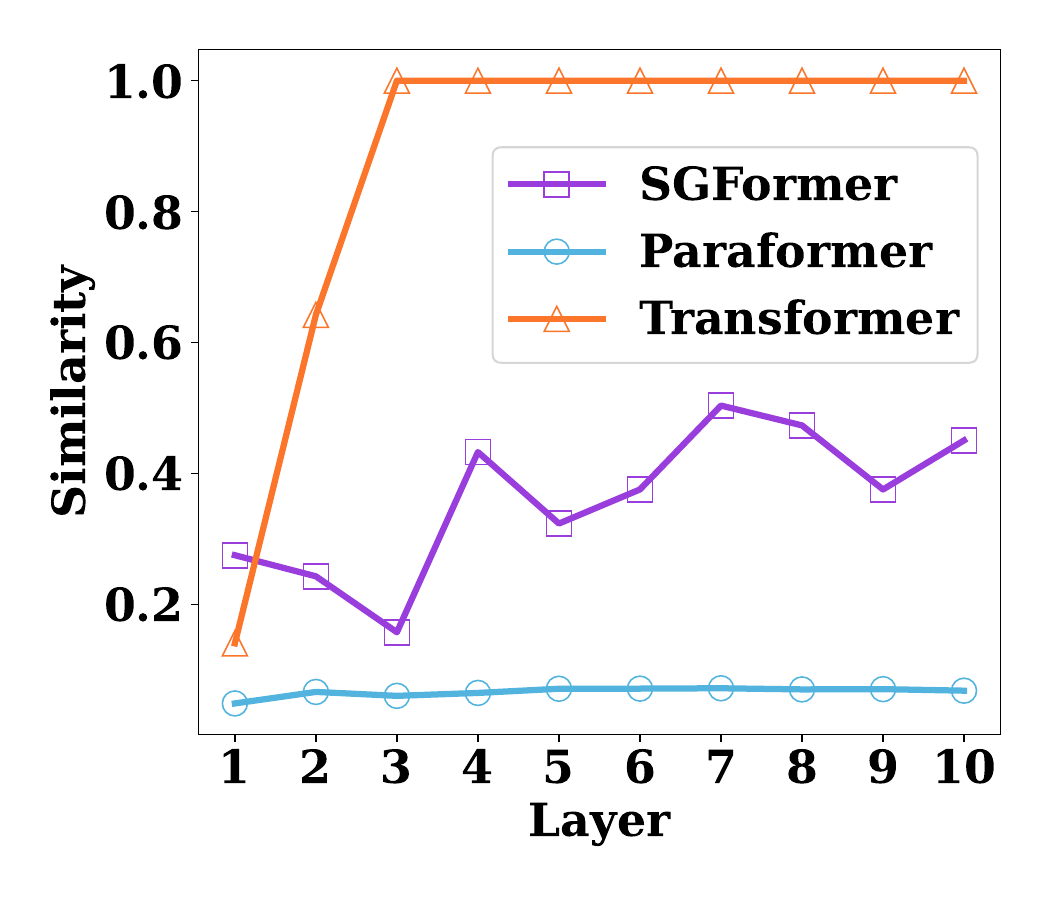}}
    \caption{The L2-distance and cosine similarity of Transformers', SGFormers' and ParaFormers' last layer representations with different number of layers or the value of K.
    }
    
    \label{fig:paraformer-oversmoothing}
\end{figure*}

\subsection{Mitigation of Over-smoothing}\label{sec.antooversm}

To validate whether ParaFormer is effective in mitigating over-smoothing, we compare the degree of over-smoothing of ParaFormer against SGFormer and vanilla Transformer.
Specifically, we train models with 1 to 10 layers on Cora and Film in node classification tasks, and compute the average pair-wise L2-distance 

between the embeddings in the final layer of all-pairs of nodes. Figure~\ref{fig:paraformer-oversmoothing} depicts the tendency of the distance and similarity of last layer's embeddings as the number of layers increases. 

The curves in Figure~\ref{fig:paraformer-oversmoothing} demonstrate that in general ParaFormer is the most resilient deep model against over-smoothing. For vanilla Transformer, the distance between node representations converges fast to 0.0 when the number of layers reaches 3 or 4, indicating that the models suffer from severe over-smoothing. 
The self-loop attention mechanism in SGFormer mitigates over-smoothing to some extent; however, models with deeper layers exhibit increasing over-smoothing.
In contrast, ParaFormer, which can be regards as Transformer of $K$ attention blocks with shared parameters, adaptively aggregates the output of each blocks into the final representation, keeping the individual node representation distinct. The results indicates that the representation distance remain stable as the number of layers increases, demonstrating ParaFormer can successfully mitigate over-smoothing.

\section{Related Work}

In this section, we will elaborate on how recent works solving the over-smoothing effects in deep GNNs and discuss the recent advances in Graph Transformers~\cite{grover20,DBLP:conf/kdd/MinLLHL23} for preventing over-smoothing.

Over-smoothing~\cite{over-smoothing18, chen2020measuring, rusch2023survey} remains an intrinsic challenge in deep GNNs, fundamentally stemming from their low-pass filtering nature. This phenomenon critically restricts the depth scalability of GNNs for capturing global graph interactions. While numerous solutions attempt to address this limitation, including architectural modifications through normalization~\cite{Zhao2020PairNorm, guo2023contranorm, scholkemper2025residual}, regularization~\cite{Dropedge-20, do2021graph, chen2023agnn, fang2023dropmessage}, or residual connections~\cite{JKNet-18, GCNII-20,huang2022tacklingoversmoothinggeneralgraph, scholkemper2025residual}, to enhance training dynamics, they predominantly rely on the \textit{local message-passing paradigm}. Consequently, these approaches inherently aggregate information from one-hop neighbors per layer, failing to effectively model long-range dependencies in large graphs~\cite{dwivedi2022long}. Transformers~\cite{Transformer} have recently emerged as a promising alternative for graph representation learning~\cite{GraphTrans21, SAT-22, GraphGPS-22, ANS-GT-22, yan2024llm, qin2024llm, zhu2025hhgt}, leveraging their all-pair attention mechanism to potentially circumvent deep GNN stacking. Various strategies integrate topological awareness through techniques, such as positional encodings~\cite{SAN-21, Graphformer21}, structural attention~\cite{Graphformer21, deng2024polynormer}, or hybrid GNN-Transformer designs~\cite{grover20, cobformer, GraphTrans21}. 

However, we identify a critical gap: \textit{Transformers intrinsically function as low-pass filters}~\cite{attention-oversmoothing-vit} and exhibit unexpected vulnerability to over-smoothing in graph domains. Although over-smoothing in Transformers has been observed and mitigated via residual connections in vision~\cite{attention-oversmoothing-vit} and language domains~\cite{attention-oversmoothing-bert}, our work reveals that \textit{even two-layer Graph Transformers generate indistinguishable node representations}. 
This fundamental limitation renders existing mitigation strategies insufficient for graph-structured data, and establishes a novel research challenge that our approach uniquely addresses.

\section{Conclusion}
In this paper, we identify that Graph Transformers face a fundamental limitation: their intrinsic low-pass filtering induces severe over-smoothing, particularly detrimental for graph tasks requiring high-frequency signals. To address this, we propose ParaFormer, a novel Graph Transformer architecture integrating a PageRank-enhanced attention mechanism. 
This design successfully addresses the over-smoothing limitation observed in prior Graph Transformer models. Furthermore, ParaFormer achieves linear computational complexity relative to graph size, enabling efficient scaling to large-scale graph datasets. The integration of PageRank requires only $\mathcal{O}(K)$ additional parameters, ensuring minimal overhead and seamless adoption within existing Graph Transformer frameworks. 
Extensive experiments across 13 benchmarks validate that ParaFormer achieves state-of-the-art performance while mitigating over-smoothing, by maintaining distinguishable representations even at depth. 
This work establishes a new paradigm for deep graph learning, bridging global structural modeling with spectral adaptability to overcome a critical obstacle in Graph Transformers.

Future works include two key directions: First, while diverse Graph Transformer variants~\cite{yuan2025survey} (e.g., using \textit{structural positional encodings} or \textit{attention bias}) continue to emerge, their susceptibility to over-smoothing and underlying mechanisms require systematic study. 
Second, ParaFormer’s capability to model both local and global dependencies while resisting over-smoothing makes it particularly valuable for real-world graphs with complex long-range interactions, such as protein interaction networks where multi-hop paths encode critical biological functions.
Therefore, we will extend ParaFormer beyond social graphs to critical domains mentioned above.

\section{Ethical Considerations}
This work presents a methodological advancement in graph learning, evaluated on standard benchmark datasets. We foresee minimal direct ethical risks as the research does not involve sensitive data or target high-stakes applications. However, we acknowledge that any graph learning technique, could potentially amplify biases or raise privacy concerns if misapplied to sensitive domains (e.g., using biased data or without privacy safeguards). We affirm compliance with the WSDM Code of Ethics.

\section*{Acknowledgment}
This work was jointly supported by the following projects: Shenzhen Science and Technology Innovation Commission under Grant JCYJ20220530143002005, Tsinghua Shenzhen International Graduate School Start-up fund under Grant QD2022024C, Shenzhen Ubiquitous Data Enabling Key Lab under Grant ZDSYS20220527171406015, the Research Grants Council of the Hong Kong Special Administrative Region, China (No. CUHK 14217622) and Damo Academy (Hupan Laboratory) through Damo Academy (Hupan Laboratory) Innovative Research Program,

\bibliographystyle{ACM-Reference-Format}
\bibliography{paraformer}

@inproceedings{DBLP:conf/kdd/MinLLHL23,
  author       = {Erxue Min and
                  Da Luo and
                  Kangyi Lin and
                  Chunzhen Huang and
                  Yang Liu},
  title        = {Scenario-Adaptive Feature Interaction for Click-Through Rate Prediction},
  booktitle    = {KDD},
  pages        = {4661--4672},
  year         = {2023}
}

@incollection{yuan2026transformer,
  title={Transformer and drug design},
  author={Yuan, Chaohao and Xu, Tingyang and Rong, Yu},
  booktitle={Deep Learning in Drug Design},
  pages={93--108},
  year={2026},
  publisher={Elsevier}
}

@inproceedings{yuan2025annotation,
  title={Annotation-guided protein design with multi-level domain alignment},
  author={Yuan, Chaohao and Li, Songyou and Ye, Geyan and Zhang, Yikun and Huang, Long-Kai and Huang, Wenbing and Liu, Wei and Yao, Jianhua and Rong, Yu},
  booktitle={Proceedings of the 31st ACM SIGKDD Conference on Knowledge Discovery and Data Mining V. 1},
  pages={1855--1866},
  year={2025}
}

@inproceedings{yuan2025non,
  title={Non-stationary equivariant graph neural networks for physical dynamics simulation},
  author={Yuan, Chaohao and Wen, Maoji and Kuruoglu, Ercan Engin and Liu, Yang and Li, Jia and Xu, Tingyang and Zhao, Deli and Cheng, Hong and Rong, Yu},
  booktitle={The Thirty-ninth Annual Conference on Neural Information Processing Systems},
  year={2025}
}

@inproceedings{DBLP:conf/iclr/LiuZCTZ0L25,
  author       = {Yang Liu and
                  Zinan Zheng and
                  Jiashun Cheng and
                  Fugee Tsung and
                  Deli Zhao and
                  Yu Rong and
                  Jia Li},
  title        = {CirT: Global Subseasonal-to-Seasonal Forecasting with Geometry-inspired
                  Transformer},
  booktitle    = {ICLR},
  year         = {2025}
}

@inproceedings{APPNP18,
  title={Predict then Propagate: Graph Neural Networks meet Personalized PageRank},
  author={Gasteiger, Johannes and Bojchevski, Aleksandar and G{\"u}nnemann, Stephan},
  booktitle={International Conference on Learning Representations},
  year={2018}
}

@inproceedings{GPRGNN21,
  title={ADAPTIVE UNIVERSAL GENERALIZED PAGERANK GRAPH NEURAL NETWORK},
  author={Chien, Eli and Peng, Jianhao and Li, Pan and Milenkovic, Olgica},
  booktitle={9th International Conference on Learning Representations, ICLR 2021},
  year={2021}
}

@inproceedings{GCN17,
  author    = {Thomas N. Kipf and
               Max Welling},
  title     = {Semi-Supervised Classification with Graph Convolutional Networks},
  booktitle = {International Conference on Learning Representations (ICLR)},
  year      = {2017}
  }

@inproceedings{GAT,
author    = {Petar Velickovic and
           Guillem Cucurull and
           Arantxa Casanova and
           Adriana Romero and
           Pietro Li{\`{o}} and
           Yoshua Bengio},
title     = {Graph Attention Networks},
booktitle = {International Conference on Learning Representations (ICLR)},
year      = {2018}
}

@inproceedings{Graphformer21,
  author    = {Chengxuan Ying and
               Tianle Cai and
               Shengjie Luo and
               Shuxin Zheng and
               Guolin Ke and
               Di He and
               Yanming Shen and
               Tie{-}Yan Liu},
  title     = {Do Transformers Really Perform Bad for Graph Representation?},
  booktitle = {Advances in Neural Information Processing Systems},
  year      = {2021}
}

@inproceedings{GraphTrans21,
  author    = {Zhanghao Wu and
               Paras Jain and
               Matthew A. Wright and
               Azalia Mirhoseini and
               Joseph E. Gonzalez and
               Ion Stoica},
  title     = {Representing Long-Range Context for Graph Neural Networks with Global
               Attention},
  booktitle = {Advances in Neural Information Processing Systems},
  year      = {2021}
  }

@inproceedings{SAT-22,
  author    = {Dexiong Chen and
               Leslie O'Bray and
               Karsten M. Borgwardt},
  title     = {Structure-Aware Transformer for Graph Representation Learning},
  booktitle = {International Conference on Machine Learning},
  year      = {2022}
  }

@inproceedings{SAN-21,
  author       = {Devin Kreuzer and
                  Dominique Beaini and
                  William L. Hamilton and
                  Vincent L{\'{e}}tourneau and
                  Prudencio Tossou},
  title        = {Rethinking Graph Transformers with Spectral Attention},
  booktitle    = {Advances in Neural Information Processing Systems},
  pages        = {21618--21629},
  year         = {2021}
}

@inproceedings{GraphGPS-22,
  author       = {Ladislav Ramp{\'{a}}sek and
                  Mikhail Galkin and
                  Vijay Prakash Dwivedi and
                  Anh Tuan Luu and
                  Guy Wolf and
                  Dominique Beaini},
  title        = {Recipe for a General, Powerful, Scalable Graph Transformer},
  booktitle = {Advances in Neural Information Processing Systems},
  pages     = {5998--6008},
  year      = {2022}
}

@article{Nodeformer22,
  title={Nodeformer: A scalable graph structure learning transformer for node classification},
  author={Wu, Qitian and Zhao, Wentao and Li, Zenan and Wipf, David P and Yan, Junchi},
  journal={Advances in Neural Information Processing Systems},
  volume={35},
  pages={27387--27401},
  year={2022}
}

@inproceedings{SGFormer23,
title={Simplifying and Empowering Transformers for Large-Graph Representations},
author={Qitian Wu and Wentao Zhao and Chenxiao Yang and Hengrui Zhang and Fan Nie and Haitian Jiang and Yatao Bian and Junchi Yan},
booktitle={Thirty-seventh Conference on Neural Information Processing Systems},
year={2023},
}

@inproceedings{he2021bernnet,
  title={BernNet: Learning Arbitrary Graph Spectral Filters via Bernstein Approximation},
  author={He, Mingguo and Wei, Zhewei and Huang, Zengfeng and Xu, Hongteng},
  booktitle={NeurIPS},
  year={2021}
}

@misc{leskovec2016snap,
  title={SNAP datasets: Stanford large network dataset collection.},
  author={Leskovec, Jure and Krevl, Andrej},
  year={2016}
}

@inproceedings{amazoncopurchase-kdd15,
  author    = {Julian J. McAuley and
               Rahul Pandey and
               Jure Leskovec},
  title     = {Inferring Networks of Substitutable and Complementary Products},
  booktitle = {{ACM} {SIGKDD} International Conference on
               Knowledge Discovery and Data Mining},
  pages     = {785--794},
  year      = {2015}
}

@article{rozemberczki2021multiscale,
  author    = {Benedek Rozemberczki and
               Carl Allen and
               Rik Sarkar},
  title     = {Multi-Scale attributed node embedding},
  journal   = {J. Complex Networks},
  volume    = {9},
  number    = {2},
  year      = {2021}
}

@article{largegraph-app-2,
  title={Improving accuracy and efficiency of blind protein-ligand docking by focusing on predicted binding sites},
  author={Ghersi, Dario and Sanchez, Roberto},
  journal={Proteins: Structure, Function, and Bioinformatics},
  volume={74},
  number={2},
  pages={417--424},
  year={2009}
}

@inproceedings{largegraph-app-3,
  author    = {Jie Tang and
               Jimeng Sun and
               Chi Wang and
               Zi Yang},
  title     = {Social influence analysis in large-scale networks},
  booktitle = {{ACM} {SIGKDD} International Conference on Knowledge Discovery and Data Mining},
  pages     = {807--816},
  year      = {2009}
}

@article{Sen08collectiveclassification,
  author    = {Prithviraj Sen and
               Galileo Namata and
               Mustafa Bilgic and
               Lise Getoor and
               Brian Gallagher and
               Tina Eliassi{-}Rad},
  title     = {Collective Classification in Network Data},
  journal   = {{AI} Mag.},
  volume    = {29},
  number    = {3},
  pages     = {93--106},
  year      = {2008}
}

@inproceedings{cobformer,
title={Less is More: on the Over-Globalizing Problem in Graph Transformers},
author={Xing, Yujie and Wang, Xiao and Li, Yibo and Huang, Hai and Shi, Chuan},
booktitle={Forty-first International Conference on Machine Learning},
year={2024},
}

@inproceedings{ANS-GT-22,
  author       = {Zaixi Zhang and
                  Qi Liu and
                  Qingyong Hu and
                  Chee{-}Kong Lee},
  title        = {Hierarchical Graph Transformer with Adaptive Node Sampling},
  booktitle = {Advances in Neural Information Processing Systems},
  year      = {2022}
}

@inproceedings{deezer20,
  author    = {Benedek Rozemberczki and
               Rik Sarkar},
  title     = {Characteristic Functions on Graphs: Birds of a Feather, from Statistical Descriptors to Parametric Models},
  booktitle = {{ACM} International Conference on Information
               and Knowledge Management},
  pages     = {1325--1334},
  year      = {2020}
  }

@inproceedings{vit2021,
title={An Image is Worth 16x16 Words: Transformers for Image Recognition at Scale},
author={Alexey Dosovitskiy and Lucas Beyer and Alexander Kolesnikov and Dirk Weissenborn and Xiaohua Zhai and Thomas Unterthiner and Mostafa Dehghani and Matthias Minderer and Georg Heigold and Sylvain Gelly and Jakob Uszkoreit and Neil Houlsby},
booktitle={International Conference on Learning Representations},
year={2021},
url={https://openreview.net/forum?id=YicbFdNTTy}
}

@inproceedings{bert-2019,
    title = "{BERT}: Pre-training of Deep Bidirectional Transformers for Language Understanding",
    author = "Devlin, Jacob  and
      Chang, Ming-Wei  and
      Lee, Kenton  and
      Toutanova, Kristina",
    editor = "Burstein, Jill  and
      Doran, Christy  and
      Solorio, Thamar",
    booktitle = "Proceedings of the 2019 Conference of the North {A}merican Chapter of the Association for Computational Linguistics: Human Language Technologies, Volume 1 (Long and Short Papers)",
    month = jun,
    year = "2019",
    address = "Minneapolis, Minnesota",
    publisher = "Association for Computational Linguistics",
    url = "https://aclanthology.org/N19-1423",
    doi = "10.18653/v1/N19-1423",
    pages = "4171--4186",
}

@inproceedings{over-smoothing18,
  title={Deeper insights into graph convolutional networks for semi-supervised learning},
  author={Li, Qimai and Han, Zhichao and Wu, Xiao-Ming},
  booktitle={Proceedings of the AAAI conference on artificial intelligence},
  volume={32},
  number={1},
  year={2018}
}

@techreport{PageRank-98,
  author = {Page, Lawrence and Brin, Sergey and Motwani, Rajeev and Winograd, Terry},
  institution = {Stanford Digital Library Technologies Project},
  title = {{The PageRank Citation Ranking: Bringing Order to the Web}},
  url = {http://citeseerx.ist.psu.edu/viewdoc/summary?doi=10.1.1.31.1768},
  year = 1998
}

@inproceedings{Transformer,
  author    = {Ashish Vaswani and
               Noam Shazeer and
               Niki Parmar and
               Jakob Uszkoreit and
               Llion Jones and
               Aidan N. Gomez and
               Lukasz Kaiser and
               Illia Polosukhin},
  title     = {Attention is All you Need},
  booktitle = {Advances in Neural Information Processing Systems},
  pages     = {5998--6008},
  year      = {2017}
}

@article{grover20,
  title={Self-Supervised Graph Transformer on Large-Scale Molecular Data},
  author={Rong, Yu and Bian, Yatao and Xu, Tingyang and Xie, Weiyang and Wei, Ying and Huang, Wenbing and Huang, Junzhou},
  journal={Advances in Neural Information Processing Systems},
  volume={33},
  year={2020}
}

@inproceedings{Dropedge-20,
title={DropEdge: Towards Deep Graph Convolutional Networks on Node Classification},
author={Yu Rong and Wenbing Huang and Tingyang Xu and Junzhou Huang},
booktitle={International Conference on Learning Representations},
year={2020},
url={https://openreview.net/forum?id=Hkx1qkrKPr}
}

@inproceedings{attention-oversmoothing-vit,
title={Anti-Oversmoothing in Deep Vision Transformers via the Fourier Domain Analysis: From Theory to Practice},
author={Wang, Peihao and Zheng, Wenqing and Chen, Tianlong and Wang, Zhangyang},
booktitle={International Conference on Learning Representations},
year={2022},
url={https://openreview.net/forum?id=O476oWmiNNp},
}

@inproceedings{attention-oversmoothing-bert,
title={Revisiting Over-smoothing in {BERT} from the Perspective of Graph},
author={Han Shi and JIAHUI GAO and Hang Xu and Xiaodan Liang and Zhenguo Li and Lingpeng Kong and Stephen M. S. Lee and James Kwok},
booktitle={International Conference on Learning Representations},
year={2022},
url={https://openreview.net/forum?id=dUV91uaXm3}
}

@article{GPR-1,
  title={Block models and personalized PageRank},
  author={Kloumann, Isabel M and Ugander, Johan and Kleinberg, Jon},
  journal={Proceedings of the National Academy of Sciences},
  volume={114},
  number={1},
  pages={33--38},
  year={2017},
  publisher={National Acad Sciences}
}

@inproceedings{GPR-2,
  title={Optimizing Generalized PageRank Methods for Seed-Expansion Community Detection},
  author={Li, Pan and Chien, I and Milenkovic, Olgica},
  booktitle={Advances in Neural Information Processing Systems},
  pages={11705--11716},
  year={2019}
}

@inproceedings{Performer-21,
  author    = {Krzysztof Marcin Choromanski and
               Valerii Likhosherstov and
               David Dohan and
               Xingyou Song and
               Andreea Gane and
               Tam{\'{a}}s Sarl{\'{o}}s and
               Peter Hawkins and
               Jared Quincy Davis and
               Afroz Mohiuddin and
               Lukasz Kaiser and
               David Benjamin Belanger and
               Lucy J. Colwell and
               Adrian Weller},
  title     = {Rethinking Attention with Performers},
  booktitle = {International Conference on Learning Representations},
  year      = {2021}
}

@inproceedings{EfficientAttentionSoftmax-21,
    author = {Zhuoran Shen and Mingyuan Zhang and Haiyu Zhao and Shuai Yi and Hongsheng Li},
    title = {Efficient Attention: Attention with Linear Complexities},
    booktitle = {WACV},
    year = {2021},
}

@inproceedings{RFM,
  author    = {Ali Rahimi and
               Benjamin Recht},
  title     = {Random Features for Large-Scale Kernel Machines},
  booktitle = {Advances in Neural Information Processing Systems},
  pages     = {1177--1184},
  year      = {2007}
}

@inproceedings{attention-based-oversmoothing-23,
title={Demystifying Oversmoothing in Attention-Based Graph Neural Networks},
author={Xinyi Wu and Amir Ajorlou and Zihui Wu and Ali Jadbabaie},
booktitle={Thirty-seventh Conference on Neural Information Processing Systems},
year={2023},
}

@inproceedings{JKNet-18,
  author    = {Keyulu Xu and
               Chengtao Li and
               Yonglong Tian and
               Tomohiro Sonobe and
               Ken{-}ichi Kawarabayashi and
               Stefanie Jegelka},
  title     = {Representation Learning on Graphs with Jumping Knowledge Networks},
  booktitle = {International Conference on Machine Learning},
  pages     = {5449--5458},
  year      = {2018}
}

@InProceedings{GCNII-20,
  title = 	 {Simple and Deep Graph Convolutional Networks},
  author =       {Chen, Ming and Wei, Zhewei and Huang, Zengfeng and Ding, Bolin and Li, Yaliang},
  booktitle = 	 {Proceedings of the 37th International Conference on Machine Learning},
  pages = 	 {1725--1735},
  year = 	 {2020},
  editor = 	 {III, Hal Daumé and Singh, Aarti},
  volume = 	 {119},
  series = 	 {Proceedings of Machine Learning Research},
  month = 	 {13--18 Jul},
  publisher =    {PMLR},
  url = 	 {https://proceedings.mlr.press/v119/chen20v.html}
}

@misc{graph-transformer-survey-architecture-22,
      title={Transformer for Graphs: An Overview from Architecture Perspective}, 
      author={Erxue Min and Runfa Chen and Yatao Bian and Tingyang Xu and Kangfei Zhao and Wenbing Huang and Peilin Zhao and Junzhou Huang and Sophia Ananiadou and Yu Rong},
      year={2022},
      eprint={2202.08455},
      archivePrefix={arXiv},
      primaryClass={cs.LG},
      url={https://arxiv.org/abs/2202.08455}, 
}

@article{newbench,
  author    = {Derek Lim and
               Xiuyu Li and
               Felix Hohne and
               Ser{-}Nam Lim},
  title     = {New Benchmarks for Learning on Non-Homophilous Graphs},
  journal   = {CoRR},
  volume    = {abs/2104.01404},
  year      = {2021}
}

@inproceedings{makingprogress,
    title={A critical look at the evaluation of {GNN}s under heterophily: Are we really making progress?},
    author={Oleg Platonov and Denis Kuznedelev and Michael Diskin and Artem Babenko and Liudmila Prokhorenkova},
    booktitle={International Conference on Learning Representations},
    year={2023}
}

@inproceedings{SGC-19,
  author    = {Felix Wu and
               Amauri H. Souza Jr. and
               Tianyi Zhang and
               Christopher Fifty and
               Tao Yu and
               Kilian Q. Weinberger},
  title     = {Simplifying Graph Convolutional Networks},
  booktitle = {International Conference on Machine Learning},
  pages     = {6861--6871},
  year      = {2019}
}

@inproceedings{h2gcn-20,
  author    = {Jiong Zhu and
               Yujun Yan and
               Lingxiao Zhao and
               Mark Heimann and
               Leman Akoglu and
               Danai Koutra},
  title     = {Beyond Homophily in Graph Neural Networks: Current Limitations and
               Effective Designs},
  booktitle = {Advances in Neural Information Processing Systems},
  year      = {2020}
}

@article{sign,
  author    = {Emanuele Rossi and
               Fabrizio Frasca and
               Ben Chamberlain and
               Davide Eynard and
               Michael M. Bronstein and
               Federico Monti},
  title     = {{SIGN:} Scalable Inception Graph Neural Networks},
  journal   = {CoRR},
  volume    = {abs/2004.11198},
  year      = {2020}
  }

@inproceedings{CPGNN,
	title={Graph neural networks with heterophily},
	author={Zhu, Jiong and Rossi, Ryan A and Rao, Anup and Mai, Tung and Lipka, Nedim and Ahmed, Nesreen K and Koutra, Danai},
	booktitle={AAAI Conference on Artificial Intelligence},
	year={2021}
}

@inproceedings{GLOGNN,
  author    = {Xiang Li and
               Renyu Zhu and
               Yao Cheng and
               Caihua Shan and
               Siqiang Luo and
               Dongsheng Li and
               Weining Qian},
  title     = {Finding Global Homophily in Graph Neural Networks When Meeting Heterophily},
  booktitle = {International Conference on Machine Learning},
  year      = {2022},
}

@inproceedings{ogb-20,
  author    = {Weihua Hu and
               Matthias Fey and
               Marinka Zitnik and
               Yuxiao Dong and
               Hongyu Ren and
               Bowen Liu and
               Michele Catasta and
               Jure Leskovec},
  title     = {Open Graph Benchmark: Datasets for Machine Learning on Graphs},
  booktitle = {Advances in Neural Information Processing Systems},
  year      = {2020}
}

@inproceedings{graphsaint,
  author       = {Hanqing Zeng and
                  Hongkuan Zhou and
                  Ajitesh Srivastava and
                  Rajgopal Kannan and
                  Viktor K. Prasanna},
  title        = {GraphSAINT: Graph Sampling Based Inductive Learning Method},
  booktitle    = {International Conference on Learning Representations},
  year         = {2020}
  }

@inproceedings{
wu2023difformer,
title={{DIFF}ormer: Scalable (Graph) Transformers Induced by Energy Constrained Diffusion},
author={Qitian Wu and Chenxiao Yang and Wentao Zhao and Yixuan He and David Wipf and Junchi Yan},
booktitle={The Eleventh International Conference on Learning Representations },
year={2023},
url={https://openreview.net/forum?id=j6zUzrapY3L}
}

@article{pedregosa2011scikit,
  title={Scikit-learn: Machine learning in Python},
  author={Pedregosa, Fabian and Varoquaux, Ga{\"e}l and Gramfort, Alexandre and Michel, Vincent and Thirion, Bertrand and Grisel, Olivier and Blondel, Mathieu and Prettenhofer, Peter and Weiss, Ron and Dubourg, Vincent and others},
  journal={the Journal of machine Learning research},
  volume={12},
  pages={2825--2830},
  year={2011},
  publisher={JMLR. org}
}

@inproceedings{zhu2003semi,
  title={Semi-supervised learning using gaussian fields and harmonic functions},
  author={Zhu, Xiaojin and Ghahramani, Zoubin and Lafferty, John D},
  booktitle={Proceedings of the 20th International conference on Machine learning (ICML-03)},
  pages={912--919},
  year={2003}
}

@article{belkin2006manifold,
  title={Manifold regularization: A geometric framework for learning from labeled and unlabeled examples.},
  author={Belkin, Mikhail and Niyogi, Partha and Sindhwani, Vikas},
  journal={Journal of machine learning research},
  volume={7},
  number={11},
  year={2006}
}

@inproceedings{jiang2019semi,
  title={Semi-supervised learning with graph learning-convolutional networks},
  author={Jiang, Bo and Zhang, Ziyan and Lin, Doudou and Tang, Jin and Luo, Bin},
  booktitle={Proceedings of the IEEE/CVF conference on computer vision and pattern recognition},
  pages={11313--11320},
  year={2019}
}

@article{zhu2020beyond,
  title={Beyond homophily in graph neural networks: Current limitations and effective designs},
  author={Zhu, Jiong and Yan, Yujun and Zhao, Lingxiao and Heimann, Mark and Akoglu, Leman and Koutra, Danai},
  journal={Advances in neural information processing systems},
  volume={33},
  pages={7793--7804},
  year={2020}
}

@misc{zheng2024graphneuralnetworksgraphs,
      title={Graph Neural Networks for Graphs with Heterophily: A Survey}, 
      author={Xin Zheng and Yi Wang and Yixin Liu and Ming Li and Miao Zhang and Di Jin and Philip S. Yu and Shirui Pan},
      year={2024},
      eprint={2202.07082},
      archivePrefix={arXiv},
      primaryClass={cs.LG},
      url={https://arxiv.org/abs/2202.07082}, 
}

@inproceedings{fang2023dropmessage,
  title={Dropmessage: Unifying random dropping for graph neural networks},
  author={Fang, Taoran and Xiao, Zhiqing and Wang, Chunping and Xu, Jiarong and Yang, Xuan and Yang, Yang},
  booktitle={Proceedings of the AAAI conference on artificial intelligence},
  volume={37},
  number={4},
  pages={4267--4275},
  year={2023}
}

@inproceedings{bambergermeasuring,
  title={On Measuring Long-Range Interactions in Graph Neural Networks},
  author={Bamberger, Jacob and Gutteridge, Benjamin and le Roux, Scott and Bronstein, Michael M and Dong, Xiaowen},
  booktitle={Forty-second International Conference on Machine Learning},
  year={2025}
}

@article{wu2021representing,
  title={Representing long-range context for graph neural networks with global attention},
  author={Wu, Zhanghao and Jain, Paras and Wright, Matthew and Mirhoseini, Azalia and Gonzalez, Joseph E and Stoica, Ion},
  journal={Advances in neural information processing systems},
  volume={34},
  pages={13266--13279},
  year={2021}
}

@article{dwivedi2022long,
  title={Long range graph benchmark},
  author={Dwivedi, Vijay Prakash and Ramp{\'a}{\v{s}}ek, Ladislav and Galkin, Michael and Parviz, Ali and Wolf, Guy and Luu, Anh Tuan and Beaini, Dominique},
  journal={Advances in Neural Information Processing Systems},
  volume={35},
  pages={22326--22340},
  year={2022}
}

@article{rusch2023survey,
  title={A survey on oversmoothing in graph neural networks},
  author={Rusch, T Konstantin and Bronstein, Michael M and Mishra, Siddhartha},
  journal={arXiv preprint arXiv:2303.10993},
  year={2023}
}

@inproceedings{chen2020measuring,
  title={Measuring and relieving the over-smoothing problem for graph neural networks from the topological view},
  author={Chen, Deli and Lin, Yankai and Li, Wei and Li, Peng and Zhou, Jie and Sun, Xu},
  booktitle={Proceedings of the AAAI conference on artificial intelligence},
  volume={34},
  number={04},
  pages={3438--3445},
  year={2020}
}

@article{radford2019language,
  title={Language models are unsupervised multitask learners},
  author={Radford, Alec and Wu, Jeffrey and Child, Rewon and Luan, David and Amodei, Dario and Sutskever, Ilya and others},
  journal={OpenAI blog},
  volume={1},
  number={8},
  pages={9},
  year={2019}
}

@inproceedings{Zhao2020PairNorm,
title={PairNorm: Tackling Oversmoothing in GNNs},
author={Lingxiao Zhao and Leman Akoglu},
booktitle={International Conference on Learning Representations},
year={2020},
url={https://openreview.net/forum?id=rkecl1rtwB}
}

@inproceedings{
deng2024polynormer,
title={Polynormer: Polynomial-Expressive Graph Transformer in Linear Time},
author={Chenhui Deng and Zichao Yue and Zhiru Zhang},
booktitle={The Twelfth International Conference on Learning Representations},
year={2024},
url={https://openreview.net/forum?id=hmv1LpNfXa}
}

@inproceedings{
chen2023nagphormer,
title={{NAG}phormer: A Tokenized Graph Transformer for Node Classification in Large Graphs},
author={Jinsong Chen and Kaiyuan Gao and Gaichao Li and Kun He},
booktitle={The Eleventh International Conference on Learning Representations },
year={2023},
url={https://openreview.net/forum?id=8KYeilT3Ow}
}

@inproceedings{zhu2025hhgt,
  title={HHGT: hierarchical heterogeneous graph transformer for heterogeneous graph representation learning},
  author={Zhu, Qiuyu and Zhang, Liang and Xu, Qianxiong and Liu, Kaijun and Long, Cheng and Wang, Xiaoyang},
  booktitle={Proceedings of the Eighteenth ACM International Conference on Web Search and Data Mining},
  pages={318--326},
  year={2025}
}

@inproceedings{
scholkemper2025residual,
title={Residual Connections and Normalization Can Provably Prevent Oversmoothing in {GNN}s},
author={Michael Scholkemper and Xinyi Wu and Ali Jadbabaie and Michael T Schaub},
booktitle={The Thirteenth International Conference on Learning Representations},
year={2025},
url={https://openreview.net/forum?id=i8vPRlsrYu}
}

@inproceedings{
guo2023contranorm,
title={ContraNorm: A Contrastive Learning Perspective on Oversmoothing and Beyond},
author={Xiaojun Guo and Yifei Wang and Tianqi Du and Yisen Wang},
booktitle={The Eleventh International Conference on Learning Representations },
year={2023},
url={https://openreview.net/forum?id=SM7XkJouWHm}
}

@article{do2021graph,
  title={Graph convolutional neural networks with node transition probability-based message passing and DropNode regularization},
  author={Do, Tien Huu and Nguyen, Duc Minh and Bekoulis, Giannis and Munteanu, Adrian and Deligiannis, Nikos},
  journal={Expert Systems with Applications},
  volume={174},
  pages={114711},
  year={2021},
  publisher={Elsevier}
}

@article{chen2023agnn,
  title={AGNN: Alternating graph-regularized neural networks to alleviate over-smoothing},
  author={Chen, Zhaoliang and Wu, Zhihao and Lin, Zhenghong and Wang, Shiping and Plant, Claudia and Guo, Wenzhong},
  journal={IEEE Transactions on Neural Networks and Learning Systems},
  volume={35},
  number={10},
  pages={13764--13776},
  year={2023},
  publisher={IEEE}
}

@article{yuan2025survey,
  title={A survey of graph transformers: Architectures, theories and applications},
  author={Yuan, Chaohao and Zhao, Kangfei and Kuruoglu, Ercan Engin and Wang, Liang and Xu, Tingyang and Huang, Wenbing and Zhao, Deli and Cheng, Hong and Rong, Yu},
  journal={arXiv preprint arXiv:2502.16533},
  year={2025}
}

@article{qin2024llm,
  title={Llm-based online prediction of time-varying graph signals},
  author={Qin, Dayu and Yan, Yi and Kuruoglu, Ercan Engin},
  journal={arXiv preprint arXiv:2410.18718},
  year={2024}
}

@article{yan2024llm,
  title={Llm online spatial-temporal signal reconstruction under noise},
  author={Yan, Yi and Qin, Dayu and Kuruoglu, Ercan Engin},
  journal={arXiv preprint arXiv:2411.15764},
  year={2024}
}

@ARTICLE{10195874,
  author={Han, Jiaqi and Huang, Wenbing and Rong, Yu and Xu, Tingyang and Sun, Fuchun and Huang, Junzhou},
  journal={IEEE Transactions on Neural Networks and Learning Systems}, 
  title={Structure-Aware DropEdge Toward Deep Graph Convolutional Networks}, 
  year={2024},
  volume={35},
  number={11},
  pages={15565-15577},
  doi={10.1109/TNNLS.2023.3288484}}

@misc{huang2022tacklingoversmoothinggeneralgraph,
      title={Tackling Over-Smoothing for General Graph Convolutional Networks}, 
      author={Wenbing Huang and Yu Rong and Tingyang Xu and Fuchun Sun and Junzhou Huang},
      year={2022},
      eprint={2008.09864},
      archivePrefix={arXiv},
      primaryClass={cs.LG},
      url={https://arxiv.org/abs/2008.09864}, 
}

\newpage
\appendix

\section{Terminology Explanation}
\label{appendix:explain}
Denote $\FT: \real^n \rightarrow \complex^n$ be the Discrete Fourier Transform. Applying $\FT$ to a sequence of nodes $\mathbf{H}$ is equivalent to left-multiplying by a DFT matrix, where the rows are the Fourier basis vectors $\Mat{f}_k = \begin{bmatrix} e^{2 \pi \mathrm{j}(k-1) \cdot 0} & \cdots & e^{2\pi \mathrm{j} (k-1) \cdot (n-1)} \end{bmatrix}^T \big/ \sqrt{n} \in \real^n$, where $k$ represents the $k$-th row of the DFT matrix, and $\mathrm{j}$ is the imaginary unit.
Let $\Mat{\tilde{z}} = \FT\Mat{z}$ represent the spectrum of $\Mat{z}$. We can split it into two parts, as $\Mat{\tilde{z}}_{dc} \in \complex$ and $\Mat{\tilde{z}}_{hc} \in \complex^{n-1}$ take the first element and the rest elements of $\Mat{\tilde{z}}$, respectively.
Formally, we can define $\DC{\Mat{z}} = \Mat{\tilde{z}}_{dc} \Mat{f}_1 \in \complex^n$ as the Direct-Current (DC) component of input $\Mat{z}$, and $\HC{\Mat{z}} = \begin{bmatrix} \Mat{f}_2 & \cdots & \Mat{f}_n \end{bmatrix} \Mat{\tilde{z}}_{hc} \in \complex^{n}$ as the complementary high-frequency component (HC).

\section{Theoretical Proofs}
\label{appendix:proof_list}

\subsection{Proof of Theorem 1}
\label{appendix:proof1}
\addtocounter{theorem}{-3}
\begin{theorem}
Given $\gamma_1 > 0 , \gamma_k=(-a)^k/2, a\in(0,\frac{1}{n}), k>1$, the attention matrix $\hat{\Mat{A}}$ in ParaFormer is a polynomial graph filter that can function as both a low-pass and high-pass graph filter.
\end{theorem}

\begin{proof}
    As defined in prior work~\cite{attention-oversmoothing-vit}, the low frequency component in attention, e.g., direct component (DC), can be written as:
    \begin{align}
        \DC{\mathbf{x}} 
        &= DFT^{-1}\text{diag}(1,0,\dots,0) DFT \mathbf{x} \\
        &= \frac{1}{n} \mathbf{1} \mathbf{1}^T \mathbf{x} 
    \end{align}
    We first prove with $\gamma_k=(-a)^k/2, a\in(0,\frac{1}{n}), k>1$, ParaFormer will only pass the high frequency information.
    Specifically, the low frequency information in ParaFormer can be represented as: $\DC{\sum_{k=0}^K \gamma_k \Mat{A}^k \mathbf{H}} = \frac{1}{n} \mathbf{1} \mathbf{1}^T \sum_{k=0}^K \gamma_k \Mat{A}^k \mathbf{H}$.

    Define $c = \mathbf{1}^T \Mat{A}$.
    Easily, we can have $\mathbf{1}^T \Mat{A}^2 = c^{(2)}$, where $c^{(m)} = \{c_0^m, c_1^m, \dots, c_n^m\}$. Due to the mathematical induction, we have: $\mathbf{1}^T \Mat{A}^k = c^{(k)}$
    \begin{align}
        \DC{\sum_{k=0}^K \gamma_k \Mat{A}^k}
        &= \frac{1}{n}\mathbf{1} \sum_{k=0}^K (-\alpha)^k \mathbf{1}^T \Mat{A}^k\\
        &= \frac{1}{n}\mathbf{1} \sum_{k=0}^K (-\alpha)^k c^{(k)}\\
        &= \frac{1}{n}\mathbf{1} \sum_{k=0}^K \{(-\alpha c_1)^{k}, (-\alpha c_2)^{k}, \dots, (-\alpha c_n)^{k}\}
        \label{eq:conidition}
        \\
        &= \mathbf{0}
    \end{align}
    As $\Mat{A}$ is an attention matrix, $c_i < n$, where $n$ is the number of nodes, then each value in column $i$ of $\Mat{A}$ is lower than $1$, due to the Softmax. We only need to guarantee $\alpha < \frac{1}{n}$, then the Equation~\ref{eq:conidition} will converge to $0$. In this way, the low frequency information will be alleviated and only contain the high frequency information after the filter.

    In another perspective, when \{$\gamma_0=1$, $\gamma_1=1$ and $\gamma_k=0, k=2,3\dots$\}, ParaFormer will be degraded as a vanilla Transformer, and has been proven as a low-pass filter by previous work~\cite{attention-oversmoothing-vit}.

    Thus, by combining two parameters of $\{\gamma\}_{\text{high}}$ and $\{\gamma\}_{\text{low}}$ which can to function as low and high filter as $\frac{1}{2}(\{\gamma\}_{\text{high}} + \{\gamma\}_{\text{low}})$, ParaFormer can adaptively pass the low, high or both frequency information in the graph.
\end{proof}

\subsection{Proof of Theorem 2}
\label{appendix:proof2}
\addtocounter{theorem}{0}
\begin{theorem}
     Given a Generalized PageRank Attention $\text{GPA}$ and a self-attention module $\text{SA}$ initialized with proper weight $\{\gamma_k\}$, we have:
     $\lambda_{\text{GPA}} \le \lambda_{\text{SA}}$.
\end{theorem}
\begin{proof}
    The proof can be directly derived from the definition of smoothing rate~\cite{attention-oversmoothing-vit}. Specifically, we have:
    \begin{align}
            \lambda_{\text{SA}} &= \sqrt{\lVert \text{Softmax}({\mathbf{P}}) \rVert_1}\lVert \mathbf{W}_V \rVert_2\\
    \end{align}
    As $\text{Softmax}({\mathbf{P}}) \in \mathbb{R}^{n\times n}$ is an attention matrix, $\sum_{i,j} \text{Softmax}({\mathbf{P}})_{ij} = n$. As $\text{Softmax}({\mathbf{P}})$ has $n$ columns, we have $\lVert \text{Softmax}({\mathbf{P}}) \rVert_1 > 1$ if not all embeddings are totally identical. Therefore, we denote $c = \lVert \text{Softmax}({\mathbf{P}}) \rVert > 1.$

    We consider a simple initialization: $\{\gamma_0 = \frac{c-1}{2}, \gamma_1 = -\frac{1}{c}, \gamma_k = 0 | k=2,3,4\cdots, K\}$.
    Given these conditions, we consider the smoothing rate of GPA:
    \begin{align}
        \lambda_{\text{GPA}} &=  \sqrt{||\sum_{k=0}^{K}\gamma_k*\text{Softmax}(\mathbf{P})^K||_1}   \lVert \mathbf{W}_V \rVert_2\\
        &=  \sqrt{||\frac{c-1}{2} \Mat{I} - \frac{1}{c} \text{Softmax}(\mathbf{P})||_1}   \lVert \mathbf{W}_V \rVert_2
    \end{align}
    At column $j$, the sum of absolute value in GPA will be:
    \begin{align}
        \sum_{i=1}^n \left| \frac{c-1}{2} \delta_{ij} - \frac{1}{c} p_{ij} \right|, 
    \end{align}
    where $\delta_{ij}$ is the element in identity matrix $\Mat{I}$.
    \begin{align}
        &\sum_{i=1}^{m} \left| \frac{c-1}{2} \delta_{ij} - \frac{1}{c} p_{ij} \right| \\ \le &\frac{c-1}{2} + \frac{1}{c} \sum_{i=1}^{m} |p_{ij}| \\ \le &\frac{c-1}{2} + \frac{1}{c} \times c \\ = &\frac{c+1}{2} < c.
        \label{eqa:softmax_1}
    \end{align}

    With conclusion in Equation~\ref{eqa:softmax_1}, we have:
    \begin{align}
        \lambda_{\text{GPA}} =&\sqrt{||\sum_{k=0}^{K}\gamma_k*\text{Softmax}(\mathbf{P})^K||_1}   \lVert \mathbf{W}_V \rVert_2 \\
        =& \max_j(\sum_{i=1}^{m} \left| \frac{c-1}{2} \delta_{ij} - \frac{1}{c} p_{ij} \right| ) \lVert \mathbf{W}_V \rVert_2 \\
        <& c \lVert \mathbf{W}_V \rVert_2 
        = \lVert \text{Softmax}({\mathbf{P}}) \rVert \lVert \mathbf{W}_V \rVert_2 \\
        =& \lambda_{\text{SA}}    
    \end{align}

    In conclusion, we have $\lambda_{\text{GPA}} < \lambda_{\text{SA}}$ with such initialization.
\end{proof}

\subsection{Proof of Theorem 3}
\label{appendix:proof3}
\begin{theorem}
    Suppose $K$ is sufficiently large, for PageRank-enhanced attention, $\hat{\mathbf{A}}^{k}\mathbf{V}, \forall k\geq k'$, will be over-smoothed. When over-smoothing happens in ParaFormer, the learnable $\gamma_k$ will converge to $0$ with appropriate learning rate.
\end{theorem}

\textbf{Proof:} First, we prove the first half of this proposition. Specifically, we aim to demonstrate that for sufficiently large $k'$, the expression $\hat{\mathbf{A}}^{k}\mathbf{V}, \forall k\geq k'$, will be over-smoothed.
This phenomenon can be interpreted through random walk theory. The attention matrix $\hat{\mathbf{A}}$ is normalized via Softmax, thereby allowing us to interpret $\hat{\mathbf{A}}$ as the probability transition matrix of a Markov chain. Each entry within this matrix is strictly positive due to the Softmax normalization, which implies the Markov chain is irreducible. Irreducibility ensures that the Markov chain converges to a unique stationary distribution denoted by $\Mat{\pi}$. Let $\Mat{a}_i$ represent the $i$-th row of $\Mat{A}$. It follows that for all $i = 1, \cdots, n$ we have $\lVert (\Mat{a}_i)^T \Mat{A} - \Mat{\pi}^T \rVert \le \lambda \lVert \Mat{a}_i - \Mat{\pi} \rVert$, where $\lambda \in (0, 1)$ is the mixing rate of the transition matrix $\Mat{A}$. Consequently, we arrive at the conclusion that, $\lim_{l \rightarrow \infty} \Mat{A}^l = \Mat{1} \Mat{\pi}^T$, which yields a low-pass filter. In the case of ParaFormer, when $k$ is sufficiently large, $\mathbf{H}_k = \mathbf{\hat{A}}^{k} \mathbf{V} = \Mat{1} \Mat{\pi}^T \mathbf{V}$. The rank of $\mathbf{H}_k$ will collapse to $1$, and all nodes will have the same representation, resulting in the over-smoothing phenomenon.

We will now proceed to demonstrate the subsequent component of our analysis. As established in the preceding paragraph, with sufficiently large $k'$, $\forall k\geq k'$, the representations $\mathbf{H}_k$ exhibit over-smoothing. In this section, we aim to show that within the framework of gradient descent, the representations $\mathbf{H}_k$ will have minimal influence on the final learned representation $\mathbf{H}$ unless the over-smoothed representations align with the final predictions.

\newcommand{\ip}[2]{\left\langle #1, #2 \right \rangle}

Formally, the label matrix is denoted as $\mathbf{Y}\in \mathbb{R}^{n\times C}$, where each row of $\mathbf{Y}$ represents a one-hot vector. The predicted probability matrix is denoted as $\hat{\mathbf{P}} \in \mathbb{R}^{n\times C}$. Accordingly, the cross-entropy loss can be defined as follows:

\begin{equation}
    L = - \sum_{i\in V} \mathbf{Y}_{i:} \log(\hat{\mathbf{P}}_{i:})
\end{equation}

For those smoothed representations, \textit{e.g.}, when $k$ is sufficiently large, we have the following limit: $\lim_{l \rightarrow \infty} \mathbf{H}_k = \mathbf{\hat{A}}^k \mathbf{V} = \Mat{1} \Mat{\pi}^T \mathbf{H}_0$. This can also be expressed as: $\mathbf{H}_k = \Mat{1} \Mat{\pi}^T \mathbf{H}_0 + o_k(1)$.
It follows that we can derive the prediction matrix as: $\hat{\mathbf{P}} = \text{Softmax}_\eta (\mathbf{H})$, where $\mathbf{H}= \sum_{k=0}^{K} \gamma_k \mathbf{H}_k$. The function $\text{Softmax}_\eta$ incorporates an additional smoothing parameter $\eta$. when $\eta=1$, $\text{Softmax}_\eta$ reduces to the standard Softmax function. We denote $\Mat{\beta}^T = \Mat{\pi}^T \mathbf{H}_0$. Consequently, we can compute the gradient with respect to $\gamma_k$:
\begin{equation}
\begin{aligned}
    \frac{\partial L}{\partial \gamma_k}
    &= \sum_{i\in \mathcal{T}} \eta (\hat{\mathbf{P}}_{i:} - \mathbf{Y}_{i:}) \mathbf{H}_{k, i:}\\
    &= \sum_{i\in \mathcal{T}} \eta (\hat{\mathbf{P}}_{i:} - \mathbf{Y}_{i:}) \Mat{\beta}^T + o_k(1)
\end{aligned}
\end{equation}

Here, we also introduce a lemma and the definition defined in GPRGNN ~\cite{GPRGNN21}.

\begin{lemma}\label{lma:softmax2argmax}
  For any real vector $\boldsymbol{\beta}\in \mathbb{R}^C$ and and sufficiently large $\eta>0$, the following holds: $\text{softmax}_{\eta}(\boldsymbol{\beta}) = \mathbf{1}[\boldsymbol{\beta}] + o_\eta(1)$, where $\mathbf{1}[\boldsymbol{\beta}] \in \mathbb{R}^C$ is defined such that $\mathbf{1}[\boldsymbol{\beta}]_{\text{argmax}(\boldsymbol{\beta})}=1$, and takes the value $0$ at all other indices.
\end{lemma}

\begin{definition}[The over-smoothing phenomenon]\label{def:oversmoothing}
To enhance clarity in the notation, in the subsequent proof we introduce a new variable to denote the final representation: $\mathbf{Z} = \sum_{k}\gamma_k \mathbf{H}^{(k)}$. If over-smoothing occurs in the ParaFormer for $K$ sufficiently large, we have $\mathbf{Z} = c_0 
\mathbf{1} \Mat{\beta}^T,\;\forall j\in [C]$ for some $c_0>0$ if $\gamma_k>0$ and $\mathbf{Z} = - c_0 
\mathbf{1} \Mat{\beta}^T,\;\forall j\in [C]$ for some $c_0>0$ if $\gamma_k<0$.
\end{definition}

Given the definition and of over-smoothing, the gradient can be reformulated as follows:

\begin{equation}
\begin{aligned}
    \frac{\partial L}{\partial \gamma_k}
    &= \sum_{i\in \mathcal{T}} \eta (\hat{\mathbf{P}}_{i:} - \mathbf{Y}_{i:}) \Mat{\beta}^T + o_k(1)\\
    &= \sum_{i\in \mathcal{T}} \eta
    (\frac{e^{\eta \mathbf{Z}_{i:}}}{\sum_{j\in[C]} e^{\eta \mathbf{Z}_{ij}}} - \mathbf{Y}_{i:}) \Mat{\beta}^T + o_k(1)\\
\end{aligned}
\end{equation}
From Definition \ref{def:oversmoothing}, we first discuss the situation of $\gamma_k > 0$. We can simplify the Softmax term via \ref{lma:softmax2argmax} and calculate the gradient as:
\begin{equation}
\begin{aligned}
    \frac{\partial L}{\partial \gamma_k}
    &= \sum_{i\in \mathcal{T}} \eta 
    (\frac{e^{\eta c_0 \Mat{\beta}^T}}{\sum_{j\in[C]} e^{\eta \mathbf{Z}_{ij}}} - \mathbf{Y}_{i:}) \Mat{\beta}^T + o_k(1)\\
    &= \sum_{i\in \mathcal{T}} \eta 
    (\mathbf{1}[c_0 \Mat{\beta}^T] - \mathbf{Y}_{i:}) \Mat{\beta}^T + o_k(1) + o_\eta(1) \\
    &= \sum_{i\in \mathcal{T}} \eta 
    \bigg(\max_{j\in [C]}\boldsymbol{\beta}^T_j - \boldsymbol{\beta}^T_{\mathbf{1}[\mathbf{Y}_{i:}]}\bigg) + o_k(1) + o_\eta(1)
\label{eqa:positive}
\end{aligned}
\end{equation}

Similarly, we can have when $\gamma_k < 0$, the gradient can be reformulated as:
\begin{equation}
\begin{aligned}
    \frac{\partial L}{\partial \gamma_k}
    &= \sum_{i\in \mathcal{T}} \eta 
    (\frac{e^{\eta c_0 \Mat{\beta}^T}}{\sum_{j\in[C]} e^{\eta \mathbf{Z}_{ij}}} - \mathbf{Y}_{i:}) \Mat{\beta}^T + o_k(1)\\
    &= \sum_{i\in \mathcal{T}} \eta 
    (\mathbf{1}[c_0 \Mat{\beta}^T] - \mathbf{Y}_{i:}) \Mat{\beta}^T + o_k(1) + o_\eta(1) \\
    &= \sum_{i\in \mathcal{T}} \eta 
    \bigg(\min_{j\in [C]}\boldsymbol{\beta}^T_j - \boldsymbol{\beta}^T_{\mathbf{1}[\mathbf{Y}_{i:}]}\bigg) + o_k(1) + o_\eta(1)
\label{eqa:negative}
\end{aligned}
\end{equation}

In conclusion, when $\gamma_k > 0$, $\frac{\partial L}{\partial \gamma_k}$ will also be positive, disregarding the $\mathcal{O}(1)$, and vice versa. In the framework of gradient descent optimization, the $\gamma_k$ will decrease if it is greater than $0$ and increase if it is less than $0$. With an appropriately chosen learning rate, $\gamma_k$ can can converge towards $0$ to alleviate the over-smoothing effect.

\section{Details of Datasets}
\label{appendix.dataset}
\begin{table}[t]
  \centering
  \caption{The details of the datasets, including the property and the number of nodes, edges, features and classes.}
    \begin{tabular}{l|r|r|r|r}
    \toprule
    Dataset & \multicolumn{1}{l|}{\# Nodes} & \multicolumn{1}{l|}{\# Edges} & \multicolumn{1}{l|}{\# Features} & \multicolumn{1}{l}{\# Classes} \\
    \midrule
    Cora & 2,708 & 5,429 & 1,433 & 7 \\
    Citeseer & 3,327 & 4,732 & 3,703 & 6 \\
    PubMed & 19,717 & 44,324 & 500   & 3 \\
    Film & 7,600 & 29,926 & 931   & 5 \\
    Squirrel & 5,201 & 216,933 & 2,089 & 5 \\
    Chameleon & 2,277 & 36,101 & 2,325 & 5 \\
    Deezer & 28,281 & 92,752 & 31,241 & 2 \\
    \midrule
    arXiv-year & 169,343 & 1,166,243 & 128   & 40 \\
    Amazon2M & 2,449,029 & 61,859,140 & 100   & 47 \\
    pokec & 1,632,803 & 30,622,564 & 65    & 2 \\
    ogbn-arxiv & 169,343 & 1,166,243 & 128   & 40 \\
    \bottomrule
    \end{tabular}
  \label{tab:dataset stats}
\end{table}

In this section, we will introduce the detailed information of these datasets. The main features of the datasets we utilize in this work are listed in Table~\ref{tab:dataset stats}.

Cora, Citeseer and Pubmed \cite{Sen08collectiveclassification} represent scientific publications as nodes, with citations between them forming edges. Node features consist of bag-of-words representations of the publication's abstract, while node labels denote the publication's subject category. Cora comprises 2,708 publications in seven computer science subfields, CiteSeer encompasses 3,327 publications across six computer science categories, and PubMed includes 19,717 publications from the PubMed database, classified into three medical categories. We employ the full-supervised settings following \cite{he2021bernnet}, which employs a random 60\%/20\%/20\% train/valid/test split.

Film \cite{largegraph-app-3} dataset encompasses a network of 7,600 actors (nodes) connected by 29,926 edges representing shared appearances on Wikipedia pages. Each actor's node is characterized by keywords extracted from their corresponding Wikipedia entry, serving as features. The dataset aims to categorize these actors into five distinct classes based on the content of their Wikipedia profiles. Notably, this graph exhibits low homophily, indicating a tendency for actors within the same category to be less interconnected compared to real-world social networks.

The Squirrel and Chameleon datasets \cite{rozemberczki2021multiscale}, derived from Wikipedia, represent page-page networks centered around specific topics. Nodes correspond to individual pages, interconnected by edges indicating mutual links. Each node's features comprise a set of informative nouns extracted from the page content. The task associated with these datasets involves classifying nodes into five categories based on their average monthly traffic. Notably, Squirrel and Chameleon exhibit significant label heterophily, meaning linked nodes often belong to different traffic categories. This inherent characteristic makes accurate classification more complex, demanding models capable of discerning subtle patterns within the network structure and node attributes. A recent work \cite{makingprogress} demonstrates there are overlapping nodes between training and testing in original Squirrel and Chameleon datasets, which will result data leakage. Furthermore, this work solve this problem and propose new splits. In our work, we thus follow the updated splits.

The Deezer-Europe \cite{deezer20} dataset, derived from the Deezer music streaming platform, presents a user-user friendship network encompassing European users. Nodes in this network symbolize individual users, with edges denoting reciprocal friendships. User preferences, specifically the artists they have `liked', constitute the node features. This dataset poses the challenge of gender prediction, aiming to classify users based on their musical tastes and connections. Following established benchmarks \cite{newbench}, the dataset is partitioned into a 50\%/25\%/25\% split for training, validation, and testing, respectively.

The ogbn-arxiv \cite{ogb-20} dataset presents a comprehensive citation network of Computer Science (CS) research papers on arXiv. This network models each paper as a node, with edges representing citations between them. Each node is characterized by a 128-dimensional feature, derived by averaging the word embeddings of the paper's title and abstract. These embeddings are generated using the WORD2VEC model. The dataset's primary task is classifying papers into one of 40 specific CS subject areas, enabling the exploration of topical relationships within the field. Following the established split \cite{ogb-20}, we train models on papers published before 2017, validate on those from 2018, and test on papers published from 2019 onwards.

The Amazon2M \cite{amazoncopurchase-kdd15} dataset, constructed from the Amazon co-purchasing network, models product relationships as a graph. Nodes in this graph symbolize individual products, while edges signify frequent co-purchasing patterns. Each product node is characterized by a bag-of-words representation derived from its textual description. Product categorization, specifically the top-level category a product belongs to, serves as the node label. Following established practices \cite{Nodeformer22}, we employed a 50\%/25\%/25\% random split for training, validation, and testing.

The Pokec \cite{leskovec2016snap} dataset offers a rich snapshot of social network interactions, comprising detailed user profiles with attributes such as geographic location, age, date of registration, and declared interests. This dataset leverages these features to predict user gender, providing a challenging classification task. Following previous work \cite{SGFormer23}, we randomly partition the dataset into training (10\%), validation (10\%), and testing (80\%) sets.

The arXiv-year \cite{newbench} dataset, derived from the ogbn-arxiv network, offers a unique perspective on scientific publications by employing publication year as the target label instead of subject areas. This dataset comprises arXiv papers as nodes, with directed edges indicating citation relationships. Each node's features are represented by averaged word2vec embeddings of the paper's title and abstract, capturing semantic information. The dataset is carefully partitioned into five classes based on publication year ranges (pre-2014, 2014-2015, 2016-2017, 2018, and 2019-2020) to ensure balanced class distribution. For the split, we employed a 50\%/25\%/25\% random partition for training, validation, and testing.

\section{Implementation Details}
\label{appendix:imple_details}
In this section, we will provide more detailed information about the implementation of ParaFormer for reproducibility.

\subsection{Notation \& Algorithm}
Table~\ref{tab:notation} provides the frequently-used notations throughout the paper. 
Algorithm~\ref{alg:scalGPA} presents the forward pass algorithm of computing the scalability GPA. Here, $K$ attention blocks share the parameter matrices $\mathbf{W}_Q$, $\mathbf{W}_K$ $\mathbf{W}_V$.
The matrix $\mathbf{M}$ caches the intermediate result of $(\mathbf{\hat{K}}^{T} \mathbf{\hat{Q}})^{(k-1)} (\mathbf{\hat{K}}^{T}\mathbf{V})$.

\begin{table}[t]
\footnotesize
\centering
\caption{Frequently Used Notations}
\label{tab:notation}

\begin{tabular}{c|c}
\toprule
Notation & Description                 \\\midrule
$\mathcal{G} = (\mathcal{V}, \mathcal{E})$ & a graph \\
$\textbf{A}$ & the adjacency matrix \\ 
$\lambda$ & the smoothing rate \\
$\gamma_k$ & the weights of GPR \\ 
$K$ & the number of layers \\ 
$\mathbf{Q}$, $\mathbf{K}$, $\mathbf{V}$ & the query, key, value matrices \\ 
$\mathbf{\hat{A}}$ & the attention matrix \\
$\mathbf{H}$/$\mathbf{X}$ & the token/node features \\
 \bottomrule
\end{tabular}

\end{table}

\subsection{ParaFormer}
To ensure a fair comparison with SGFormer, the implementation of ParaFormer$_{\text{GCN}}$ maintains consistency in the GCN architecture employed for each dataset. This encompasses utilizing the same GCN layer, hidden dimension, and the weighting assigned to the graph representation's contribution to the final hidden dimension. For these detailed hyper-parameters, please refer to the original paper and repository of SGFormer.

In the implementation of ParaFormer$_{\text{GCN}}$, we fix the hyper-parameter $K$ as 10 across all datasets.
The remaining hyperparameters are determined through a grid search strategy within the following search space:
\begin{itemize}
    \item learning rate within $\{0.005, 0.01, 0.05, 0.1\}$.
    \item hidden dimension within $\{64, 96, 128, 256\}$.
    \item dropout rate within $\{0.3, 0.4, 0.5, 0.6, 0.7\}$.
\end{itemize}

In the implementation of ParaFormer$_{\text{GPRGNN}}$, we set the $K$ as 10 both for ParaFormer and GPRGNN. In our implementation, instead of paralellizing the ParaFormer and GNN block, we turn to another more natural approach to integrate the structural information. Specifically, we sum up the GPR enhanced attention matrix with GPR enhanced adjacency matrix as: $\mathbf{H} = \sum_{k=0}^{10} (\lambda_k \mathbf{\hat{A}}^k + \lambda_{k}^{'}\mathbf{A}^{k})\mathbf{V}.$
As other hyper-parameters, we search the hyper-parameters within the space:
\begin{itemize}
    \item learning rate within $\{0.001, 0.005, 0.01, 0.05\}$.
    \item weight decay within $\{0.0, 1e-4, 5e-4, 1e-3\}$.
    \item hidden dimension within $\{64, 128, 256, 512\}$.
    \item dropout rate within $\{0.3, 0.5, 0.7, 0.8\}$.
    \item graph weight within $\{0.3, 0.5, 0.6, 0.7, 0.8, 0.9\}$.
\end{itemize}
Following the SGFormer, we adopt a full-batch training approach for the medium-sized graphs, ogbn-arxiv and arxiv-year. For the larger datasets, Amazon2M and pokec, a mini-batch training method is utilized. The batch size is set as 0.1M, which is the same as SGFormer. Furthermore, we also utilize the graph partitioning strategy employed in SGFormer.

\end{document}